\documentclass[twoside,11pt]{article}

\usepackage{blindtext}
\usepackage{comment}
\newcommand{\shuffle}{\raisebox{\depth}{\rotatebox{270}{$\exists$}}}
\usepackage[abbrvbib, preprint]{jmlr2e}
\usepackage[english]{babel}
\usepackage{xcolor}
\usepackage{comment}
\usepackage{booktabs}
\usepackage{amsmath}
\usepackage{bbm}
\usepackage{multirow}

\renewcommand{\hat}[1]{\widehat{#1}}

\newcommand{\vertiii}[1]{{\left\vert\kern-0.25ex\left\vert\kern-0.25ex\left\vert #1 
    \right\vert\kern-0.25ex\right\vert\kern-0.25ex\right\vert}}

\newcommand{\R}{\mathbb{R}}

\newcommand{\Sig}[1]{\mathrm{Sig}{(#1)}}

\newcommand{\dsig}[2]{\varrho^{\mathrm{Sig}}(#1,#2)}

\newcommand{\dsigtrunc}[3]{\varrho_{\leq #3}^{\mathrm{Sig}}(#1,#2)}

\usepackage{mathrsfs}
\usepackage{mathtools}
\usepackage{atbegshi}
\usepackage[autostyle]{csquotes}
\usepackage[toc,page]{appendix}

\usepackage{lastpage}

\ShortHeadings{}{C. Bayer, D. Gogolashvili, L. Pelizzari}
\firstpageno{1}

\begin{document}

\title{Local Regression on Path Spaces with Signature Metrics}

\author{\name Christian Bayer \email bayerc@wias-berlin.de \\
       \addr Weierstrass Institute (WIAS)\\
       Berlin, Germany\\
       \AND
       \name Davit Gogolashvili \email davit.gogolashvili@wias-berlin.de \\
       \addr Weierstrass Institute (WIAS)\\
       Berlin, Germany\\ 
       \AND
       \name Luca Pelizzari \email 
       luca.pelizzari@univie.ac.at \\
       \addr
       University of Vienna \\
       Vienna, Austria
       }

\editor{My editor}

\maketitle

\begin{abstract}
    We study nonparametric regression and classification for path-valued data. We introduce a functional Nadaraya-Watson estimator that combines the signature transform from rough path theory with local kernel regression. The signature transform provides a principled way to encode sequential data through iterated integrals, enabling direct comparison of paths in a natural metric space. Our approach leverages signature-induced distances within the classical kernel regression framework, achieving computational efficiency while avoiding the scalability bottlenecks of large-scale kernel matrix operations. We establish finite-sample convergence bounds demonstrating favorable statistical properties of signature-based distances compared to traditional metrics in infinite-dimensional settings. We propose robust signature variants that provide stability against outliers, enhancing practical performance. Applications to both synthetic and real-world data—including stochastic differential equation learning and time series classification—demonstrate competitive accuracy while offering significant computational advantages over existing methods. 
\end{abstract}


\tableofcontents

\section{Introduction}
Many supervised learning problems involve \emph{path-valued input data}—observations that take the form of sequential or temporal processes. Examples include financial asset prices evolving over time \citep{bouchaud2018trades}, physiological signals such as EEG or ECG \citep{hannun2019cardiologist,schirrmeister2017deep}, handwritten character trajectories \citep{graves2007unconstrained}, and human action recognition \citep{yang2022developing}. Unlike the classical setting, which typically assumes fixed-dimensional vector inputs, path-valued data present unique challenges: they are often irregularly sampled, vary in length, and may exhibit strong temporal dependencies. 

To address these challenges, one promising approach is the \emph{signature transform}, first developed in \cite{chen1957integration}. It provides a systematic way to extract features from sequential data by encoding a path through its iterated integrals, thereby summarizing essential information in a tensorial form. Crucially, this representation enables direct comparison of sequences of varying size and length. The signature transform further possesses several remarkable properties that make it particularly well-suited for machine learning:
\vspace{-0.2cm}
\begin{itemize}
    \setlength{\itemsep}{0pt}
    \setlength{\parskip}{0pt}
    \item the signature naturally encodes geometrical properties of the path and is invariant under time-reparametrization;
    \item linear functionals of the signature are universal approximators on path space;
    \item for intermeditate to long time series, the signature can provide remarkable compression;
\end{itemize}
\vspace{-0.2cm}
Hence, the use of signatures in statistical learning has received considerable attention. While we refer to the review article by \cite{lyons2022signature} for a broad overview, we focus on discussing the most directly related contributions.
Within the machine learning literature, signatures are most commonly regarded as a feature extraction technique, a perspective that has been applied in a wide range of practical applications. Extending this view, \cite{morrill2020generalised} introduce the Generalised Signature Method, which provides a unifying framework for variations of the signature transform in multivariate time series analysis.

The literature most closely related to our work concerns signature methods for functional regression. In the parametric setting, linear functional regression with signatures—both with and without regularization—has been well-established in the literature \citep{fermanian2021embedding,fermanian2022functional,bleistein2023learning,cohen2023nowcasting,guo2025consistency,bayer2025primal}. 
There, the regression functional is assumed to be linear in the signature with finitely many unknown coefficients, a representation motivated by the Stone–Weierstrass theorem. Fully nonparametric approaches are represented by works on signature kernels \citep{kiraly2019kernels,chevyrev2018signature,lee2023signature,schell2023nonparametric,horvath2023optimal, bayer2025pricing}. While signature kernels can be computed efficiently \citep{kiraly2019kernels,salvi2021signature,lemercier2024log,toth2023random}, these methods inherit the well-known scalability limitations of kernel learning, in particular the computational burden associated with inverting large Gram matrices. The same limitation applies to related Bayesian approaches, such as Gaussian process regression with signature kernel covariance functions \citep{toth2020bayesian}. To address the scalability issue, \citet{lemercier2021siggpde} proposed a sparse variational inference framework for Gaussian processes. Alternatively, signatures can also be used as a feature within a deep learning pipeline \citep{kidger2019deep, moreno2024rough}.  

The field of nonparametric statistics for functional data represents a well-established area of research. The work of \cite{ferraty2006nonparametric} provides a comprehensive treatment of kernel-based methods for functional data analysis, establishing the theoretical foundations for nonparametric regression when the input space is infinite-dimensional. Subsequent developments include convergence rate analysis \citep{lian2012convergence, 10.3150/15-BEJ709}, variable selection methods \citep{aneiros2022variable,shang2014bayesian} and extensions to more complex functional structures \citep{selk2023nonparametric}.

Within the framework of kernel-based functional regression, the choice of \emph{semi-metric} plays a crucial role in determining both theoretical properties and practical performance.  Classical approaches typically rely on $L^p$ metrics or Hölder norms. However, these standard choices often fail to reflect the intrinsic geometric structure of path-valued data. In fact, the associated topologies invariably produce \emph{small-ball probabilities} of exponential form—whether based on the supremum norm, $L_p$, or Hölder norms \citep{ferraty2006nonparametric}—which limits their ability to enhance concentration properties. This motivates the exploration of \emph{semi-metrics}, which can be designed specifically for sequential data and provide a more natural framework for comparing paths.

\paragraph{Contributions.} In this paper, we propose a new estimator that leverages signature transforms within the classical local kernel (or Nadaraya–Watson) regression framework. Our contributions are twofold: (i) we establish rigorous finite-sample convergence guarantees for the proposed estimator, addressing a gap in the theoretical understanding of nonparametric methods based on signatures, and (ii) we provide an efficient and straightforward implementation that avoids the computational and scalability challenges commonly associated with signature kernel approaches, demonstrating its practical applicability on real-world datasets.

The remainder of the paper is organized as follows. Section 2 presents the background material along with the proposed estimator. Section 3 is devoted to our main theoretical results, including the convergence analysis for signature-based Nadaraya-Watson estimators and the derivation of convergence rates. Section 4 provides detailed experimental validation on synthetic and real-world datasets.

\section{Notations and Background}
We consider the problems of regression and classification for data $(X,Y) \in \mathcal{X}\times \mathcal{Y}$, where $\mathcal{X}$ is a (infinite-dimensional) path-space, and $\mathcal{Y}$ is some finite-dimensional Euclidean space. We assume that the data samples are drawn from a joint probability measure $P$ on $\mathcal{X}\times \mathcal{Y}$.

Our central object of interest is the \emph{regression function}, given by
\begin{equation}\label{eq:regression_intro} 
F(x) = \mathbb{E}[Y|X=x] = \int ydP(y|x), \quad x \in \mathcal{X}, \end{equation}
where $P(\cdot | x)$ denotes the conditional distribution of $Y$ given $X = x.$ For instance, in financial modeling one may view the data $X = (X_t)_{t \geq 0} \in \mathcal{X}$ as the dynamics of an asset price. Then, for any \emph{option payoff} $Y = \Psi(X)$, the conditional expectation \eqref{eq:regression_intro} represents its fair price.


For the classification problem, we consider \emph{categorical} responses for the path-valued data, that is we are interested in conditional probabilities 
\begin{equation}\label{eq:classification}
 p_g(x) = \mathbb{P}[Y=g|X=x], \quad (x,g) \in \mathcal{X}\times \mathcal{Y}.
\end{equation} A typical classification example is the handwriting recognition problem, where $\mathcal{Y}$ is the alphabet $\{A,B,\dots,Z\}$, and the data $X\in \mathcal{X}$ may be seen as handwritings of such letters - represented as paths in $\mathbb{R}^2$. The function $p_g$ then assigns the probability of such a handwriting to correspond to the letter $g\in \mathcal{Y}$.

Both the regression and classification problems reduce to learning a conditional expectation, so their treatment follows the same principle, which we now briefly outline. Assume we have i.i.d. (independent and identically distributed) data of input-output pairs
\[
(X^{(1)}, Y^{(1)}),\dots, (X^{(M)},Y^{(M)}) \overset{\text{i.i.d.}}{\sim} P.
\]
Motivated by the classical Nadaraya-Watson estimate \citep{nadaraya1964estimating,watson1964smooth}, and in particular the functional extensions thereof \citep{ferraty2006nonparametric}, we consider the estimator for the regression \eqref{eq:regression_intro}

\begin{equation}\label{eq:NW_estimator_regression}
    \hat{F}(x) =\frac{\sum_{i=1}^MY^{(i)}K(h^{-1}\varrho(x,X^{(i)}))}{\sum_{j=1}^MK(h^{-1}\varrho(x,X^{(j)})}, 
\end{equation} 
where $\varrho$ is a semi-metric\footnote{Satisfying all properties of a metric except for point-separation, i.e. $\varrho(x,y)=0 \nRightarrow x=y$.}
 on the path-space $\mathcal{X}$, $K:\R \to \R_+$ some asymmetric kernel, and $h=h(M)$ is strictly positive. The same estimator can be applied to the classification problem \eqref{eq:classification} by replacing $Y$ with $1_{\{Y=g\}}$, since the latter can be viewed as a regression problem with target $p_g(x) = \mathbb{E}[1_{\{Y=g\}}|X = x]$. It is well-known in the literature (see, e.g. \cite[Chapter 13]{ferraty2006nonparametric}), that compared to the finite-dimensional case, the choice of the semi-metric is very sensible from both a theoretical and practical point of view. In the following section, we introduce variants of the estimator \eqref{eq:NW_estimator_regression}, where the semi-metric $\varrho$ is defined via the \emph{signature transform} of the data, $\mathcal{X} \ni x \mapsto \Sig{x}$; see \eqref{eq:sig_distance}–\eqref{eq:trunc_distance}.

\section{Local regression with signatures}\label{sec:local_sig_regr}

This section presents the main theoretical results of the article, starting with a general convergence analysis of the estimator \eqref{eq:NW_estimator_regression}.
\subsection{Convergence analysis}\label{sec:convergence:analysis}
As is customary in nonparametric regression, we impose smoothness constraints on the regression function $F$. The choice of these constraints is particularly delicate in the infinite-dimensional setting, where the topology of the underlying space $\mathcal{X}$ plays a crucial role in determining both the convergence rates and the practical performance of the estimator.
\begin{definition} \label{Holder class}
    Let $\beta \in (0,1]$ and $L$ be any positive constant. For any semi-metric $\varrho$ on $\mathcal{X}$, we denote by $\mathcal{F}_{\beta}^{\varrho}$ the Hölder class of functionals $F: \mathcal{X} \rightarrow \mathbb{R}$ that satisfy the condition
    \begin{equation}\label{Lipschitz condition}
    \left|F(x)-F\left(x^{\prime}\right)\right| \leq L\varrho (x,x^{\prime})^{\beta},
    \end{equation}
    for all $x, x^{\prime} \in \mathcal{X}$.
\end{definition}
\begin{remark}
   It is important to note that our convergence analysis applies to smoothness levels $\beta \leq 1$. Since the domain $\mathcal{X}$ is not a vector space in general, extending to higher smoothness levels presents significant challenges due to the absence of a natural notion of differentiation. 
\end{remark}

A fundamental quantity in the convergence analysis is the {\em concentration function}, which measures how the data concentrates around a given point in the metric space. 
\begin{definition}
    For any semi-metric $\varrho$ on $\mathcal{X}$ and $x \in \mathcal{X}$, we define the concentration function
\begin{equation}
\label{eq:small_ball}
    \phi^{\varrho}_x(h) = \mathbb{P}[\varrho(X, x) \leq h], \quad h \geq 0.
\end{equation}
\end{definition}

The behaviour of $\phi^\varrho_x(h)$ as $h \downarrow 0$ determines the convergence rates of our estimator. Intuitively, slower decay of $\phi^\varrho_x(h)$ indicates that the data is less dispersed around $x$, leading to better statistical performance. This behavior is intimately connected to the geometry of the path space and the choice of distance function.

In the classical finite-dimensional setting, when $\mathcal{X} = \R^d,$ the concentration function typically exhibits polynomial decay $\mathcal{O}(h^{d})$ regardless of the choice of the metric (as in finite-dimensional spaces all norms are equivalent). However, in infinite-dimensional spaces, the situation is more delicate and depends heavily on the choice of metric. 

The choice of the semi-metric $\varrho$ in Definition \ref{Holder class} is of paramount importance, as it directly influences both the class of admissible functions $\mathcal{F}_{\beta}^{\varrho}$ and the small-ball probability behaviour $\phi^\varrho$ that governs the convergence rates. For brevity, we will often write $\mathcal{F}_\beta = \mathcal{F}_\beta^{\varrho}$ and $\phi_x = \phi_x^{\varrho}$ whenever the choice of $\varrho$ is fixed and clear from the context.

 In the setting of path-valued data, we propose using signature-based distances in Section \ref{sec:sig_metrics}, which naturally respect the geometric structure of the underlying paths. Before doing so, we establish the fundamental convergence rates of the Nadaraya–Watson estimator with respect to \emph{any} semi-metric $\varrho$ on $\mathcal{X}$.


\begin{theorem}\label{main_theorem}
    Let $Y \in [-R,R]$ and let $F \in \mathcal{F}_{\beta}$ with smoothness parameter $\beta \in (0,1]$.
    Consider the estimator $\hat{F}$ defined in \eqref{eq:NW_estimator_regression}, and assume that the kernel $K$ is compactly supported and satisfies
    \begin{equation}\label{thm:kernel}
    b 1_{[0,1]} \leq K \leq B 1_{[0,1]},
    \end{equation}
    for some constants $0 < b \leq B < \infty$. For any $\delta \in (0,1)$ and $M$ satisfying
    \[
    M \geq \frac{16 B \log(6/\delta) }{b \phi_x(h)},
    \]
    with probability at least $1-\delta$ the following bound holds
    \begin{equation}
        |\hat{F}(x) - F(x)| \leq  \frac{B}{b} h^{\beta} + 8R\sqrt{\frac{2 B \log(6/\delta)}{b \phi_x(h) M}}.
    \end{equation}

\end{theorem}

The proof of the theorem can be found in Appendix \ref{sec:appendix}. Theorem \ref{main_theorem} provides a finite-sample bound of the pointwise error for the Nadaraya-Watson estimator, which decomposes the estimation error into two fundamental components: the {\em bias} term that is directly controlled by the Hölder parameter $\beta$ from Definition \ref{Holder class}, reflecting how the local smoothness of $F$ around the target point $x$ affects the estimation quality. And the {\em variance} term $\mathcal{O}\left((\phi_x(h)M)^{-1/2}\right),$
that captures the stochastic fluctuations due to finite sample size $M$.

In the Euclidean case, when $\mathcal{X} \subseteq \R^d,$ the concentration function exhibits polynomial decay $\phi_x(h) \sim Ch^d,$ for some constant $C>0.$ Assuming an optimal choice of bandwidth $h \sim M^{-1/(2\beta+d)}$  that balances the bias-variance trade-off, Theorem \ref{main_theorem} yields the classical nonparametric rates of convergence $M^{-\beta/(2\beta+d)}$, that is known to be optimal in the minimax sense \citep{stone1980}.

However, in infinite-dimensional spaces, the situation becomes significantly more delicate and the convergence behavior fundamentally changes. The concentration function typically exhibits exponential behavior $\phi_x(h) \sim C \exp\left(-c h^{-\gamma}\right).$ Gaussian processes provide an illuminating example. For instance, for fractional Brownian motion $X^H=(X^H_t:0\leq t \leq 1)$ we have $\phi_x(h) \sim C\exp(-h^{-2/(2H-\beta)})$ when $\varrho$ is chosen to be the $\beta$-Hölder distance \citep{li2001gaussian}. 

Exponential decay of the concentration function leads to convergence rates that are fundamentally slower than their finite-dimensional counterparts. Specifically, under regularity conditions on the regression function $F \in \mathcal{F}_\beta$ and appropriate choice of bandwidth, one can achieve slow logarithmic type rates $\mathcal{O}(\log(M)^{-2\beta/\gamma}),$ known to be optimal in the minimax sense \citep{10.3150/15-BEJ709}. In the following subsection, we analyze the concentration function behavior for two specific choices of signature-based distances that are particularly relevant for rough path data.

\subsection{Signature semi-metrics on path-spaces}\label{sec:sig_metrics}

Our main focus in the applications below is on \emph{path-valued data}, that is, on spaces $\mathcal{X}$ consisting of paths $x:[0,T]\to E$, where the state space is typically the Euclidean space $E=\mathbb{R}^d$. In this section, we propose a canonical metric choice induced by the signature transform, which offers several theoretical and practical advantages over conventional distances on path spaces.

For simplicity of exposition, we assume in this section that $\mathcal{X} = C^1([0,T],\mathbb{R}^d)$, whereas a more general construction for \emph{rougher data}—namely $\mathcal{X} = C^{p-var}([0,T],\mathbb{R}^d)$ with $p\geq 1$—is discussed in Appendix~\ref{app:signature_intro}. The signature of a path $x \in \mathcal{X}$ is given as a sequence of tensors
\[
\Sig{x} = \left (1,\Sig{x}^{(1)},\dots, \Sig{x}^{(k)}, \dots \right ) \in \prod_{k\geq 0}(\mathbb{R}^d)^{\otimes k},
\] consisting of iterated integrals 
\[
\Sig{x}^{(k)} = \left (\int_0^T \int_0^{t_k} \cdots \int_0^{t_2} dx^{i_1}_{t_1} \cdots dx^{i_k}_{t_k}\right )_{ i_1,\dots,i_k \in [d]},
\]
where $[d] = \{1,\dots,d\}$, see also Definition \ref{def:sig_appendix}.

Among the many fascinating algebraic and analytical properties of the signature transform, some of which we summarize in Appendix \ref{app:signature_intro}, the one most relevant for this article is that the sequence $\operatorname{Sig}(x)$ can be regarded as an \emph{encoding} or \emph{description} of the trajectory. Indeed, the transform $x \mapsto \operatorname{Sig}(x)$ is injective up to some equivalence class $\sim$ (see Appendix \ref{app:signature_intro} for more details), so that the sequence $\operatorname{Sig}(x)$ uniquely characterizes the underlying path $x$. Such equivalence classes become trivial once paths are augmented with time, $x_t \mapsto (t,x_t)$; see \eqref{eq:augmented_space} and the preceding discussion. At the cost of increasing the dimension by one, we henceforth assume that 
\[
   \mathcal{X} = \widehat{C}^1 = \{x_t=(t,x_t): x\in C^{1}([0,T],\R^d)\}.
\]
Injectivity of the signature transform allows us to introduce a \emph{Euclidean-like} distance between paths
\begin{equation}\label{eq:sig_distance}
\varrho^{\mathrm{Sig}}: \mathcal{X} \times \mathcal{X} \rightarrow \mathbb{R}_+, \quad \dsig{x}{y} = \Vert \Sig{x}-\Sig{y} \Vert,
\end{equation} where \[
\Vert \mathbf{a} \Vert = \sqrt{ \sum_{k\geq 0}  \Vert \mathbf{a}^{(k)} \Vert_{(\mathbb{R}^d)^{\otimes k}}^2}, \quad \mathbf{a} \in \prod_{k\geq 0} (\mathbb{R}^d)^{\otimes k},
\]  see Lemma \ref{lem:metric_appendix}. In words, the distance between two data points $x$ and $y$ is measured by comparing their signatures in the \emph{extended tensor algebra} $\mathcal{T} = \prod_{k \geq 0} (\mathbb{R}^d)^{\otimes k}$. In Appendix~\ref{app:signature_intro} we provide a more detailed introduction to the algebraic structure of $\mathcal{T}$, including its product $\otimes$ and addition $+$, as well as further details on its Hilbert space structure.

\begin{remark}\label{rem:factorial}
A more conventional distance on $C^1$ is given by the $1$-variation (or length) of the difference of two paths, namely  
\[
   \varrho^{1}(x,y) = \|x-y\|_{1\text{-var}}, 
   \quad 
   \|x\|_{1\text{-var}} = \int_0^T |\dot{x}_t|\,dt.
\] We know $\Vert \Sig{x}^{(k)}\Vert_{(\mathbb{R}^d)^{\otimes k}} \leq \frac{\|x\|_{1\text{-var}}^k}{k!}$ from Lemma~\ref{lemma:decay}, so that 
the distance in \eqref{eq:sig_distance} is finite.

\end{remark}
A key feature of \eqref{eq:sig_distance} on the infinite-dimensional space $\mathcal{X}$ is that it naturally admits finite-dimensional projections, obtained by truncating the sequence $\operatorname{Sig}(x)$ at some tensor level $N$.
 Supported by the factorial decay noted in Remark~\ref{rem:factorial}, for $N$ large enough the distance \eqref{eq:sig_distance} is well-approximated by its truncated version
\begin{equation}\label{eq:trunc_distance}
    \dsigtrunc{x}{y}{N}
   = \sqrt{ \sum_{n=0}^N \big\| \Sig{x}^{(n)} - \operatorname{Sig}^{(n)}(y) \big\|^2_{(\mathbb{R}^d)^{\otimes n}} }.
\end{equation}
 
\begin{remark}
    Both the truncated and untruncated signature distances are well-supported by publicly available open-source libraries, such as iisignature \citep{reizenstein2018iisignature} or roughpy \citep{morley2024roughpy}. In particular, for the untruncated distance one can exploit its relation to the \emph{signature kernel}
\[
   \dsig{x}{y}^2 = k(x,x) - 2k(x,y) + k(y,y),
\]
where $k$ is obtained by solving a Goursat-type PDE; \citep{salvi2021signature}.

\end{remark}

Returning to the local regression analysis from Section~\ref{sec:convergence:analysis}, we now derive convergence rates for the truncated signature distance. Before doing so, we show in the following lemma that in both the truncated and untruncated cases, the rate of convergence is always at least as fast as under the $1$-variation metric; the proof can be found in Appendix \ref{app:proofs_signature}.

\begin{lemma}\label{lem:small_ball_comp}
    For any $\mathcal{R}>0$ and random variable $X$ in $\mathcal{X}_\mathcal{R}= \{x\in \mathcal{X}: \Vert x \Vert_{1-var} \leq \mathcal{R}\}$, we have \[
    \phi_x^{\varrho^{\mathrm{Sig}}_{\leq N}}(h) \geq \phi_x^{\varrho^\mathrm{Sig}}(h) \geq \phi_x^{\varrho^1}(Ch), \quad \forall x \in B_\mathcal{R},
    \] for some constant $C>0$. 
\end{lemma}



To derive convergence rates, we recall from Appendix~\ref{app:signature_intro} that the signature takes values in a free nilpotent Lie group $\mathcal{G} \subset \mathcal{T}$, which will be crucial for the following assumption.

\begin{Assumption}\label{ass:density}
    We assume that $X$ is a random variable taking values in $\mathcal{X}$, and its truncated signature
\[
   \Sig{X}^{\leq N} = \big(1,\,\operatorname{Sig}^{(1)}(X),\,\dots,\,\operatorname{Sig}^{(N)}(X)\big)
\]
admits a density function $p$ with respect to the Haar measure of the Lie group $\mathcal{G}^{\leq N}$, see Definition \ref{def:haar_measure_lie}, which is bounded away from zero, i.e., $p(\mathbf{g}) \geq c > 0$ for all $\mathbf{g} \in \mathcal{G}^{\leq N}$.
\end{Assumption}
\begin{example}\label{ex:Brownian}
    While we restrict to smooth random paths $X$ here for simplicity, Assumption~\ref{ass:density} remains reasonable in rougher frameworks. For instance, it has been shown to hold for Brownian motion $X=B$ already in \cite{kusuoka1987applications}, which is relevant for our application in Section~\ref{sec:SDE_learning}, and has further been extended to fractional Brownian motion $X = B^H$ with $H>1/4$ recently in \cite{baudoin2020density}.

\end{example}
\begin{proposition}\label{prop:truncated_samllball}
    Under Assumption~\ref{ass:density}, the small-ball probability with respect to the truncated signature distance~\eqref{eq:trunc_distance} has at most polynomial decay, in the sense that for any $N \in \mathbb{N}$
\[
   \phi^{\varrho_{\leq N}^{\mathrm{Sig}}}_x(h) = \mathbb{P}\left[\dsigtrunc{X}{x}{N} \leq h \right] 
\geq C h^{\nu(N)}, \quad x \in \mathcal{X},
\]
where $C>0$ is constant and
\[
   \nu(N) = \sum_{n=1}^N\frac{1}{n} \sum_{\ell \mid n} \mu\!\left(\tfrac{n}{\ell}\right) d^{\ell},
\]
and $\mu(\cdot)$ denotes the Möbius function, and the inner sum is taken over divisors of $n$.

\end{proposition}
\begin{remark}
    While the rigorous proof is deferred to Appendix \ref{app:proofs_signature}, we briefly indicate why such a result is natural. Owing to the density assumption, the small-ball probability is bounded below by the volume of balls $B_h$ in $\mathcal{G}^{\leq K}$. We will see that these volumes scale as $h^{\dim(\mathfrak{g}^{\leq K})}$, where $\mathfrak{g}^{\leq N}$ is the Lie algebra associated with $\mathcal{G}^{\leq N}$, whose homogeneous dimension is in fact given by $\dim(\mathfrak{g}^{\leq N})=\nu(N)$, see \cite[Theorem 6]{reutenauer2003free}.

\end{remark}
    As a consequence, we obtain the following non-parametric convergence rate for our estimator \eqref{eq:NW_estimator_regression} with respect to the truncated signature distance, the proof can be found in Appendix \ref{app:proofs_signature}.

\begin{corollary}\label{cor:truncated_rates}
   Let $Y$ be a fixed random variable in $[-R,R]$ and assume $$F(x) = \mathbb{E}[Y| X = x] \in \mathcal{F}^{\varrho}_\beta, \quad \varrho = \varrho_{\leq N}^{\mathrm{Sig}},$$ for some $N\in \mathbb{N}$. Suppose that the kernel $K$ satisfies \eqref{thm:kernel} and  Assumption~\ref{ass:density} holds true, and set $h=M^{-1/(2\beta+\nu(N))}$. For $\delta >0$ and $M$ large enough, we can find a constant $C=C(\delta,R)$ such that 
   \[
      \mathbb{P}\Big[|\hat{F}(x) - F(x)| \leq  C M^{-\frac{\beta}{2\beta+\nu(N)}}\Big ] \geq 1-\delta.
   \]
\end{corollary}
    For fixed $N$, the previous result yields Euclidean-type nonparametric rates in the infinite-dimensional space $\mathcal{X}$, governed by the effective dimension $\nu(N)$. We emphasize once more that the choice of semi-metric not only affects the convergence rate, but also determines the class of admissible functions $\mathcal{F}_{\beta}^{\varrho}$, which is expected to be much smaller compared to the full signature distance. While it might not always be justified that the underlying functions lie in $\mathcal{F}^\varrho_\beta$, we observe excellent performance when using the truncated signature distance in all our applications. 

\begin{remark}\label{rem:Brownian_signature}
   While we have only considered smooth data in this section, i.e. $\mathcal{X} = \widehat{C}^1$, typical stochastic processes such as Brownian motion have more irregular sample paths. A more suitable framework is the space of paths of finite $p$-variation with $p \geq 2$; see Definition~\ref{def:p_var}. As outlined in Remark~\ref{rem:rough_case}, by exploiting rough path theory \citep{lyons1998differential}, the signature remains well defined using Stratonovich iterated integrals \cite[Chapter 13]{friz2010multidimensional}
   \[
      \Sig{\hat{B}}^{i_1 \cdots i_n}
      = \int_0^T \int_0^{t_n} \cdots \int_0^{t_2}
      \circ d\hat{B}^{i_1}_{t_1} \cdots \circ d\hat{B}^{i_n}_{t_n}.
   \]
 The induced distance $\varrho_{\mathrm{Sig}}$—and the theory developed in this article—remains valid on rough path spaces.

\end{remark}

\subsection{Solution maps of rough differential equations}\label{sec:RDE}
In this section, we briefly excurs into an application of our results to a relevant class of path-valued mappings in stochastic analysis, namely, the solution maps of rough differential equations. This, in particular, means that we now deal with rougher data—more involved than the setting considered so far—namely, rough path spaces \(\mathcal{X} = \mathscr{C}_g^{\alpha}([0,T];\mathbb{R}^d)\). We postpone the proof of the main result to Appendix~\ref{app:proofs_signature} and refer the interested reader to the excellent textbook \cite{friz2020course} for background on this topic.

The mappings of interest are the solution maps $\mathbf{x} \mapsto \mathcal{I}(\mathbf{x}) = Y_T$, where \(Y\) solves the rough differential equation (RDE) \cite[Chapter~8]{friz2020course}
\[
Y_0 \in \mathbb{R}^d, \qquad dY_t = \sigma(Y_t)\,d\mathbf{x}_t, \quad 0 < t \leq T,
\]
with coefficients \(\sigma \in C_b^3\) and geometric rough drivers \(\mathbf{x} \in \mathscr{C}^\alpha_g\) for some \(\alpha \in (1/3, 1/2)\); see \cite[Chapter~3]{friz2020course}. 

For technical reasons, related to the boundedness condition for the targets \(Y\) in Theorem~\ref{main_theorem}, 
we aim to learn $\mathcal{I}$ locally on $$\mathcal{X}_R = \left \{\mathbf{x} \in \mathscr{C}^\alpha:  \vertiii{\mathbf{x}}_{\alpha;[0,T]}\leq R \right \},$$ for some \(R > 0\). 
Moreover, in the main result below we also rely on the enhanced Cameron-Martin space (see, e.g., \cite[Chapter 13.5]{friz2010multidimensional}) \begin{equation}
    \label{eq:CM_space}
    \mathscr{H} = \left \{\mathbf{x}= \Sig{x}^{\leq 2}: x \in \mathcal{H} \right \} \subset \mathscr{C}^\alpha, \quad \mathcal{H} = \left \{ \int_0^\cdot \dot{x}_t dt: \dot{x} \in L^2([0,T];\mathbb{R}^d) \right \}.
\end{equation}
\begin{remark}
    We can randomize the RDE solution using Brownian rough paths $\mathbf{x}=\mathbf{B}(\omega)$ \cite[Chapter 3.2]{friz2020course}\begin{equation}\label{eq:Brownian:RP}t \mapsto {\mathbf{B}}_{0,t} = \left (1, B_{t}, \int_0^t {B}_s \otimes \circ d {B}_s\right ) \in \mathcal{G}^{\leq 2}.\end{equation} Then, the target $Y_T(\omega) = \mathcal{I}(\mathbf{B}(\omega))$ almost surely coincides with the terminal value to the Stratonovich SDE \begin{equation}\label{eq:Strat}
        Y_0=Y_0 \in \mathbb{R}^d, \quad dY_t = \sigma(Y_t) \circ dB_t, \quad 0<t \leq T.
    \end{equation}
    Since we restrict the learning problem to the ball $\mathcal{X}_R$ for some fixed $R>0$, without loss of generality we replace $\mathbf{B}$ with the stopped Brownian rough path $\mathbf{B}^R$ \begin{equation}\label{eq:stopped_RP}
\mathbf{B}^R_{0,t}(\omega)= \left (1,B^R_t, \int_0^tB^R_s\otimes \circ dB_s^R \right )(\omega) \in \mathcal{X}_R, \quad B^R(\omega) = B_{t \land T_R(\omega)}(\omega),
    \end{equation} where $T_R(\omega) = \inf\{t \geq 0 : \vertiii{\mathbf{B}(\omega)}_{\alpha; [0,t]} > R\} \wedge T$.
\end{remark}
Returning to our local regression setting, the previous remark suggests that we can learn $\mathcal{I}$ via regression, when generating i.i.d.\ input–output pairs  
\[
X^{(m)} = \mathbf{B}^{R,(m)}, \qquad Y^{(m)} = \mathcal{I}(\mathbf{B}^{R,(m)}), \quad m = 1, \dots, M,
\]
where, in practice, \( Y^{(m)} \) is obtained by solving the SDE~\eqref{eq:Strat}, for instance via an Euler scheme; see also Section~\ref{sec:SDE_learning}. Similar as before, we define the estimator \begin{equation}\label{eq:estimator_RP}
    \hat{\mathcal{I}}(\mathbf{x})= \frac{\sum_{i=1}^M 
\mathcal{I}(\mathbf{B}^{R,(i)})K(h^{-1}\dsig{\mathbf{x}}{\mathbf{B}^{R,(i)}})}{\sum_{j=1}^M K(h^{-1}\dsig{\mathbf{x}}{\mathbf{B}^{R,(j)}})},
\end{equation} where the signature distance \eqref{eq:sig_distance} can be generalized to the space $\mathcal{X}=\mathscr{C}_g^{\alpha}$; see Appendix \ref{app:signature_intro}.

\begin{theorem}\label{thm:RDE}
 Suppose that $\mathcal{I} \in \mathcal{F}^{\varrho}_{\beta}$ for $\varrho = \varrho_{\mathrm{Sig}}$, where $\mathcal{I}$ is the Itô-Lyons map\[
    \mathcal{I}:\mathcal{X}_R \rightarrow \mathbb{R}, \quad \mathbf{x}|_{[0,T]} \mapsto \mathcal{I}(\mathbf{x}) =Y_T,
    \] for some $R>0$ and assume the kernel $K$ satisfies \eqref{thm:kernel}. Then, for any $\delta \in (0,1)$ and $h = \left (\frac{\log(M)}{\hat{K}}\right )^{{\alpha-
    1/2}}$ for $M$ large enough and some constant $\hat{K}>0$, we can find another constant $C=C(\delta,R)>0$ such that \[
    \mathbb{P}\left [|\mathcal{I}(\mathbf{x})-\hat{\mathcal{I}}(\mathbf{x})| \leq C \log(M)^{-\zeta} \right]\geq 1-\delta, \qquad \zeta = {\beta}({1/2-\alpha}), \quad \mathbf{x} \in \mathcal{X}_{R'} \cap \mathscr{H},  
    \] for any $0 < R' < R$.
\end{theorem}
The proof can be found in Section \ref{app:proofs_signature}. Let us conclude this section with several remarks.
\begin{remark}
\begin{itemize}
    \item[(i)] It should be noted that the space $\mathscr{H}$ lies dense in $\mathscr{C}^{\alpha}$ with respect to the $\vertiii{\cdot}_{\alpha}$-topology, see  \cite[Theorem 13.55 and Remark 19.4]{friz2010multidimensional}. As a consequence, under the same assumptions as in Theorem \ref{thm:RDE}, for any $\mathbf{x}\in \mathcal{X}_R$ and any $\epsilon \in (0,1)$, we can find $\mathbf{x}_0 \in \mathcal{X}_{R'} \cap \mathscr{H}$ such that $\vertiii{\mathbf{x}-\mathbf{x}_0}_\alpha \leq \epsilon C\log(M)^{-\zeta}$ and \[\mathbb{P}\left [|\mathcal{I}(\mathbf{x})-\hat{\mathcal{I}}(\mathbf{x}_0)| \leq C \log(M)^{-\zeta} \right] \geq \mathbb{P}\left [|\mathcal{I}(\mathbf{x}_0)-\hat{\mathcal{I}}(\mathbf{x}_0)| \leq (1-\epsilon)C \log(M)^{-\zeta} \right] \geq 1-\delta,\] for $M$ large enough.
    \item[(ii)] Let us note that the only condition to be verified in Theorem~\ref{thm:RDE} is that the solution map \(\mathcal{I}\) is Hölder continuous with respect to the signature metric \(\varrho_{\mathrm{Sig}}\) for some \(\beta \in (0,1]\). This is, in particular, satisfied for RDE solutions that admit a signature expansion of the form
\[
    Y_T = \langle \ell, \mathrm{Sig}(\mathbf{x})_T \rangle, 
    \qquad \ell \in (\mathfrak{T}, \|\cdot\|),
\]
see also Appendix~\ref{app:signature_intro}. For RDEs, we refer to \cite[Chapter 20.4.2]{friz2010multidimensional} for  such Taylor expansions in the signature. In the context of SDEs, this relates to stochastic Taylor expansions \cite{arous1989flots,kloeden1991stratonovich}, which have recently been considered in the context of infinite signature expansions in 
\cite{cuchiero2023signature,jaber2024path}, with precise conditions ensuring their convergence.

\item[(iii)] If we replace \(\varrho_{\mathrm{Sig}}\) by the \(\alpha\)-Hölder rough path distance \cite[Definition~2.4]{friz2020course}
\[
    \varrho_\alpha(\mathbf{x}, \mathbf{y}) 
    = |x_0 - y_0| + \vertiii{\mathbf{x} - \mathbf{y}}_{\alpha; [0,T]},
    \qquad \mathbf{x}, \mathbf{y} \in \mathscr{C}^\alpha([0,T]; \mathbb{R}^d),
\] then the condition 
\(\mathcal{I} \in \mathcal{F}_\beta^{\varrho_\alpha}\) is locally satisfied by the Lipschitz continuity of the Itô–Lyons map; see, for example, \cite[Theorem~8.5]{friz2020course}. Moreover, building on the small-ball probability analysis for Gaussian rough paths in \cite{salkeld2022small}, it is also possible to replace the Brownian rough path drivers \(\mathbf{B}\) by more general Gaussian rough paths; see \cite[Chapter~10]{friz2020course}.

\end{itemize}

\noindent
Both the concrete conditions ensuring \(\mathcal{I} \in \mathcal{F}_\beta^{\varrho_{\mathrm{Sig}}}\) and the extensions to broader classes of rough paths and applications, will be addressed in the forthcoming paper \cite{bayer2025nonparametric}.

\end{remark}
%


\subsection{Robustification}\label{sec:robustification}
In our numerical experiments in Section~\ref{sec:applications}, we observe that the estimators $\hat{F}$ are not robust to samples yielding unusually large signature values, leading to outliers in the predictions and unstable performance. Indeed, if a testing sample produces a big signature entry, its distance to a “typical” signature sample becomes very large, and consequently the estimator~\eqref{eq:NW_estimator_regression} predicts a value close to zero. For Brownian signatures, see also Remark \ref{rem:Brownian_signature}, it is not difficult to anticipate this phenomenon, say for $T=1$, since the signature contains the powers
\[
   \int_0^1 \int_0^{t_n} \cdots \int_0^{t_2} \circ d\hat{B}^{2}_{t_1}\cdots \circ d\hat{B}^{2}_{t_n}
   \;=\; \frac{B_1^n}{n},
\]
where $B_1 \sim \mathcal{N}(0,1)$. With small probability (Gaussian tail-estimates), these entries can become arbitrarily large whenever $B_1 \gg 1$. We illustrate and further comment on this observation in Figure \ref{fig:learning_SDE} in Section \ref{sec:SDE_learning}.


We found that this issue can be addressed by adopting the robust signature proposed in \cite{chevyrev2018signature}, which we summarize below; further details are provided at the end of Appendix~\ref{app:proofs_signature}, an explicit construction of $\Lambda$ is provided in Example \ref{ex:tensor_normalization}. The robust signature, following \cite{chevyrev2018signature}, is defined by $\operatorname{RSig}(x) = \Lambda \circ \operatorname{Sig}(x)$,
where $\Lambda$ is a tensor normalization (Definition \ref{def:tensor_normalization})
\[
   \Lambda : \mathcal{T} \;\longrightarrow\; \left\{\mathbf{a} \in \mathcal{T} : \|\mathbf{a}\| \leq R \right\},
\]
for some fixed $R>0$. The map $\Lambda$ is required to be continuous and injective in order to preserve the structural properties of the signature transform. 

\section{Application}\label{sec:applications}
In this section, we test our estimator~\eqref{eq:NW_estimator_regression} in two applications. First, we learn the solution map of stochastic differential equations as a functional of the driving noise, formulated as a nonparametric regression problem. Second, we apply the method to classification tasks on various real-world datasets consisting of sequential data, which we interpret as piecewise linear paths.

All signature computations are performed using the \texttt{iisignature} library \citep{reizenstein2018iisignature}, which provides efficient implementations of signature algorithms with computational complexity $\mathcal{O}(Ld^N)$ where $L$ is the path length, $d$ is the dimension, and $N$ is the truncation level. 
\begin{figure}[ht]
\centering
\includegraphics[width=1\linewidth]{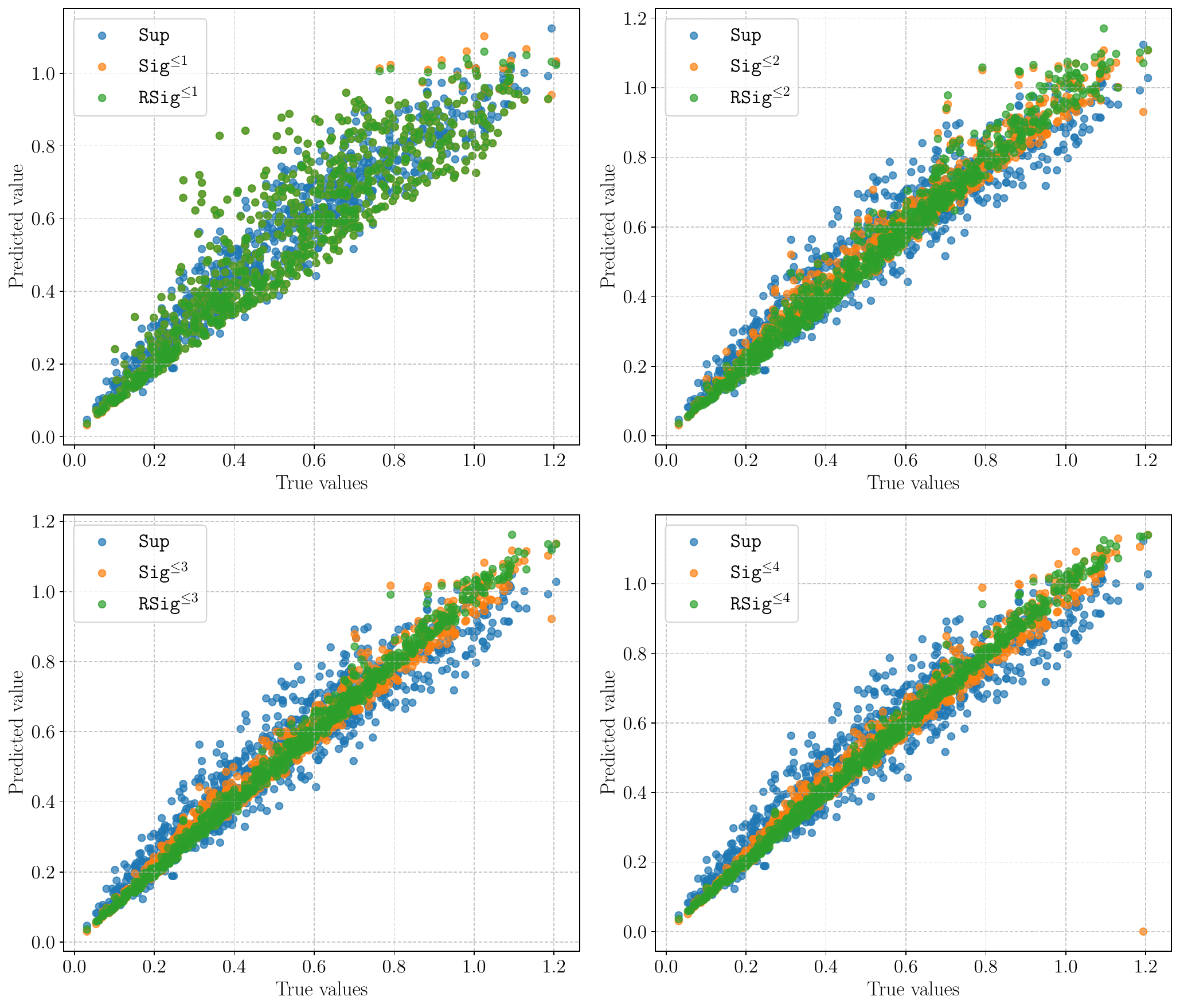}
    \caption{Scatter plots of the testing data $\{(Y^{(m)}, \hat{Y}^{(m)}):m \in I_{te} \}$, using signature and supremum metrics in \eqref{eq:estimator_SDE}. At truncation level $N=4$, the point ${\color{orange} \bullet}$ near $(1.2,0)$ illustrates an outlier of the basic method \texttt{Sig}; see the discussion in Section~\ref{sec:robustification}.}

    \label{fig:learning_SDE}
\end{figure}
Our basic method, which we call \texttt{Sig}, uses the estimator defined in equation \eqref{eq:NW_estimator_regression} with the truncated signature-based semi-metric from equation \eqref{eq:trunc_distance}. Throughout the experiments involving local regression, we use the standard Gaussian kernel. The \texttt{RSig} method uses robust signature features to improve stability and discriminative power. For each time series, we compute truncated signatures up to level $N$ and apply a robust transformation $\Lambda \circ \operatorname{Sig}$, where the normalization map $\Lambda$ rescales each signature tensor according to its magnitude (see Appendix~\ref{app:signature_intro}, Example~\ref{ex:tensor_normalization}). This rescaling reduces sensitivity to outliers by dampening the influence of large signature components \citep{chevyrev2018signature}.

We select hyperparameters (bandwidth $h$, robust parameters $C$ and $a$) via cross-validation. The signature level is set according to the time series dimension (capped at 5), and the bandwidth and robustness parameters are selected from predefined grids.

\subsection{Learning the solution map of SDEs}\label{sec:SDE_learning}
Let $(B_t)_{t\in [0,T]}$ be an $m$-dimensional Brownian motion, and consider its time-augmentation $\hat{B}_t=(t,B_t)$. We are interested in the \emph{Itô-map} $\hat{B} \mapsto Z_T$, where $$Z_0 = z_0, \quad dZ_t= b(Z_t)dt + \sigma(Z_t)dB_t, \quad 0<t \leq T,$$ with coefficients $b:[0,T]\times \mathbb{R}^d \rightarrow \mathbb{R}^d$ and $\sigma:[0,T] \times \mathbb{R}^d \rightarrow \mathbb{R}^{d \times m}$ sufficiently regular such that a unique strong solution exists, see, e.g., \cite[Chapter 3, Theorem 7]{P05}.              

For the numerical experiments in this section, we fix $m=d=1$ and consider the smooth coefficients
\[
   b(x) = -x^{p}, 
   \quad 
   \sigma(x) = x \cos(x), \quad p \in \mathbb{N},
\] and choose $p=5$. For some $M \in \mathbb{N}$, we draw $B^{(1)},\dots,B^{(M)}$ independent Brownian sample paths on some grid $\{t_0,\dots,t_L\}$ in $[0,T]$, and denote by $Y^{(m)}$ the terminal values $Y^{(m)}= Z^{(m)}_T$, obtained using an Euler-Mayurama scheme. We split the data into disjoint training and testing sets $I_{tr} \dot{\cup} I_{te}= \{1,\dots,M\}$ and for $m\in I_{te}$ consider
\begin{equation}\label{eq:estimator_SDE}
   \hat{Y}^{(m)}=
   \frac{\sum_{i\in I_{tr}} Y^{(i)} K\!\left(h^{-1}\varrho(\hat{B}^{(m)},\hat{B}^{(i)})\right)}
         {\sum_{j \in I_{tr}} K\!\left(h^{-1}\varrho(\hat{B}^{(m)},\hat{B}^{(j)})\right)}.
\end{equation}

In Figure~\ref{fig:learning_SDE} we plot the testing data $\{(Y^{(m)},\hat{Y}^{(m)}): m\in I_{te}\}$, choosing $M=2^{13}$ and a $90\%-10\%$ split, i.e. $M_{tr} = |I_{tr}|=0.9 \times M$ and $M_{te} = |I_{te}| = 0.1\times M$. The estimator in \eqref{eq:estimator_SDE} is evaluated using  $\texttt{Sig}$ and $\texttt{RSig}$ at increasing truncation levels, as well as \texttt{Sup}, which uses the conventional supremum distance $\rho_{\mathrm{sup}}(x,y) = \sup_t |x_t - y_t|$. Although both signature-based distances visibly outperform the supremum distance, we observe that $\texttt{Sig}$ is sensitive to outliers, see the discussion in Section \ref{sec:robustification}. 

\begin{table*}[t]
\centering
\resizebox{\textwidth}{!}{%
\begin{tabular}{lrrrrrrrrrrr} \toprule $M_{tr}$ & $\texttt{RSig}^{\leq 2}$ & $\texttt{Sig}^{\leq 2}$ & $\texttt{RSig}^{\leq 3}$ & $\texttt{Sig}^{\leq 3}$ & $\texttt{RSig}^{\leq 4}$ & $\texttt{Sig}^{\leq 4}$ & \texttt{$2$-var} & \texttt{$4$-var} & $L^1$ & $L^2$ & $\texttt{Sup}$ \\ \midrule 8 & 0.115 & 0.146 & 0.114 & 0.163 & 0.114 & 0.172 & 0.206 & 0.323 & 0.197 & 0.190 & 0.177 \\ 16 & 0.088 & 0.130 & 0.090 & 0.150 & 0.090 & 0.160 & 0.190 & 0.161 & 0.211 & 0.219 & 0.166 \\ 32 & 0.060 & 0.111 & 0.061 & 0.128 & 0.061 & 0.138 & 0.179 & 0.213 & 0.183 & 0.164 & 0.154 \\ 64 & 0.062 & 0.126 & 0.126 & 0.127 & 0.071 & 0.135 & 0.180 & 0.142 & 0.175 & 0.160 & 0.139 \\ 128 & 0.054 & 0.104 & 0.051 & 0.111 & 0.051 & 0.119 & 0.164 & 0.129 & 0.141 & 0.141 & 0.121 \\ 256 & 0.047 & 0.056 & 0.042 & 0.060 & 0.039 & 0.069 & 0.167 & 0.128 & 0.124 & 0.121 & 0.104 \\ 512 & 0.042 & 0.055 & 0.033 & 0.064 & 0.032 & 0.061 & 0.163 & 0.124 & 0.111 & 0.106 & 0.093 \\ 1024 & 0.041 & 0.051 & 0.029 & 0.050 & 0.027 & 0.050 & 0.164 & 0.093 & 0.104 & 0.097 & 0.088 \\ 2048 & 0.040 & 0.050 & 0.026 & 0.048 & 0.025 & 0.048 & 0.165 & 0.087 & 0.100 & 0.091 & 0.085 \\ 4096 & 0.039 & 0.040 & 0.023 & 0.035 & 0.021 & 0.042 & 0.164 & 0.080 & 0.098 & 0.087 & 0.083 \\ 8192 & 0.038 & 0.040 & 0.021 & 0.034 & 0.019 & 0.041 & 0.165 & 0.074 & 0.098 & 0.084 & 0.083 \\ \midrule time & 15.29 & 0.86 & 18.68 & 1.25 & 23.94 & 2.20 & 1840.87 & 1704.02 & 24.85 & 22.52 & 33.54 \\ \bottomrule \end{tabular}%
}
\caption{SDE regression accuracy (RMSE) on the testing data. 
The last row corresponds to the evaluation time (in seconds) required 
for the whole procedure for the largest training sample size $M_{tr}=8192$. 
Note that we do not include $1$-var here, since almost surely $\Vert B \Vert_{1-var}=+\infty$.}
\label{tab:regression_results}
\end{table*}

Finally, Table~\ref{tab:regression_results} illustrates the convergence studied in Theorem~\ref{main_theorem} and Corollary~\ref{cor:truncated_rates} as the number of samples $M$ increases. In addition to \texttt{Sig}, \texttt{RSig} and \texttt{Sup}, we also include the methods \texttt{$p$-var} and $\texttt{L}^p$ based on the metrics induced by classical $L^p$- and $p$-variation norms (see \eqref{eq:p_var_norm} in Appendix \ref{app:signature_intro}). For each training sample size, the table reports the root mean-squared error (RMSE)
evaluated on the independent testing set of size $M_{te}= |I_{te}| =8192$. The last row additionally presents the maximal running time (in seconds) of the procedure, including the hyperparameter optimization and the evaluation of the RMSE for each method for tha largest training sample size $M_{tr}=8192$.

The results clearly show that \texttt{Sig} and \texttt{RSig} substantially outperform the estimators based on conventional metrics, both in terms of accuracy and computational efficiency. Among the signature methods, \texttt{RSig} consistently achieve better performance. The sensitivity of \texttt{Sig}, discussed in Section~\ref{sec:robustification}, becomes apparent as its accuracy worsens when the truncation level increases from $3$ to $4$, whereas \texttt{RSig} continues to improve. The price to pay lies in the additional computational effort required for the robustification. Nevertheless, it remains highly efficient—comparable to the ``simpler'' methods $\texttt{L}^p$ and \texttt{Sup} —whereas \texttt{$p$-var} methods become impractical for this task.

\subsection{Time Series Classification}
We evaluate our signature-based Nadaraya-Watson classifier on the UEA time series classification archive \citep{bagnall2018uea}\footnote{We follow the dataset subset used in the original work \cite{bagnall2018uea}, rather than selecting a subset ourselves. This ensures consistency and allows for a reproducible comparison with prior results.}. Our experimental setup includes comparisons with some distance-based time series classifiers \citep{abanda2019review}, such as dynamic time warping (\texttt{DTW}),  canonical signature pipeline (\texttt{SigP})
with a random forest classifier \citep{morrill2020generalised} and an analysis of different distance metrics within the local regression framework.

Table~\ref{tab:classification_results} presents the classification accuracy (expressed as percentages) across 21 datasets. We make the following key observations. \texttt{Sup} and \texttt{L}$^2$ serve as natural baselines for our methods, since all share the same kernel regression framework and differ only in the choice of distance. Across datasets, we observe that \texttt{Sup} and \texttt{L}$^2$ frequently yield lower accuracy. This demonstrates that the signature distance offers a more expressive and robust similarity measure for sequential data than pointwise metrics.

\texttt{DTW}-based methods remain strong performers on some datasets (e.g., Cricket, Handwriting), reflecting their effectiveness when local temporal shifts are the dominant source of variability. However, both \texttt{Sig} and \texttt{RSig} achieve comparable or better accuracy on many datasets without requiring explicit alignment.

The computational complexity of \texttt{Sig} is of order $\mathcal{O}((M_{te} + M_{tr}) L d^N + M_{te} M_{tr} d^N)$
which is linear in sequence length $L$ and linear in the number of training samples $M_{tr}$, with the exponential dependence on $d^N$ controlled by choosing small truncation levels (typically $N \leq 5$). \texttt{DTW} methods on the other hand, scales quadratically in $L$, making it unpractical for a long time series. For the comparison, on the Ethanol dataset, which has the longest sequence length among our benchmark datasets, \texttt{DTW}-based classification requires approximately 1018.49 seconds for complete evaluation (on CPU). In comparison, our signature-based method with fixed bandwidth and truncation level $N = 5$ completes in just 4.89 seconds (also on CPU, without GPU acceleration). 

\begin{table}[!htb]
\centering
\begin{small}
\begin{tabular}{l@{\hspace{4pt}}c@{\hspace{4pt}}c@{\hspace{4pt}}c@{\hspace{5pt}}c@{\hspace{5pt}}c@{\hspace{5pt}}c@{\hspace{5pt}}c}
\toprule
\textbf{Dataset} & $\texttt{DTWD}$ & $\texttt{DTWA}$ & $\texttt{SigP}$ & $\texttt{Sup}$ & $L^2$ & $\texttt{RSig}$ & $\texttt{Sig}$ \\
\midrule
ArticularyWord & \textbf{98.7} & \textbf{98.7} & 97.7 & 92.3 & 94.3 & 97.0 & 95.7 \\
AtrialFibrillation & 20.0 & 26.7 & \textbf{46.7} & 33.3 & 33.3 & 26.7 & 33.3 \\
BasicMotions & 97.5 & \textbf{100} & \textbf{100} & 52.5 & 67.5 & 95.0 & 92.5 \\
Cricket & \textbf{100} & \textbf{100} & 95.8 & 58.3 & 81.9 & 83.3 & 75.0 \\
Epilepsy & 96.4 & \textbf{97.8} & 95.7 & 47.8 & 53.6 & 73.1 & 44.9 \\
Ethanol & 32.3 & 31.6 & 43.3 & 25.1 & 25.1 & \textbf{38.8} & 24.7 \\
Ering & 91.5 & 92.6 & \textbf{94.8} & 81.9 & 87.8 & 81.1 & 62.6 \\
FaceDetection & 52.9 & 52.8 & \textbf{61.4} & 50.9 & 52.5 & 51.5 & 51.3 \\
FingerMovements & \textbf{53.0} & 51.0 & 52.0 & 49.0 & 49.0 & 48.0 & 49.0 \\
HandMovement & 18.9 & 20.3 & 20.3 & 20.3 & 20.3 & \textbf{29.7} & 27.0 \\
Handwriting & \textbf{60.7} & \textbf{60.7} & 37.9 & 17.4 & 12.6 & 28.7 & 12.1 \\
Heartbeat & 71.7 & 69.3 & 69.8 & 72.2 & \textbf{74.6} & 71.7 & 72.6 \\
Libras & 87.2 & 88.3 & \textbf{93.9} & 82.8 & 71.1 & 82.8 & 77.8 \\
LSST & 55.1 & 56.7 & \textbf{56.9} & 10.4 & 16.7 & 41.3 & 31.5 \\
NATOPS & 88.3 & 88.3 & \textbf{92.2} & 75.6 & 76.1 & 81.7 & 76.1 \\
PenDigits & \textbf{97.7} & \textbf{97.7} & 97.4 & 91.4 & 96.3 & 88.4 & 95.1 \\
Racketsports & 80.3 & 84.2 & \textbf{90.8} & 56.6 & 82.2 & 79.6 & 74.3 \\
SCP1 & 77.5 & 78.5 & \textbf{78.8} & 50.2 & 50.2 & 77.1 & 68.9 \\
SCP2 & \textbf{53.9} & 52.2 & 50.6 & 50.0 & 51.1 & \textbf{53.9} & 52.2 \\
StandWalkJump & 20.0 & 33.3 & \textbf{46.7} & \textbf{46.7} & 13.3 & 33.3 & 33.3 \\
UWaveGesture & 90.3 & 90.0 & \textbf{90.9} & 82.5 & 84.4 & 83.8 & 80.6 \\
\bottomrule
\end{tabular}
\caption{Classification accuracy (\%) comparison across 21 benchmark time series datasets. }
\label{tab:classification_results}
\end{small}
\end{table}

\section{Conclusion}
In this paper, we have proposed a signature-based Nadaraya-Watson estimator for nonparametric regression and classification on path spaces. We provide a rigorous finite-sample guarantee.

Experimental validation on SDE learning and time series classification demonstrates consistent improvements over conventional distance metrics. While signature-based methods may not always outperform specialized techniques like DTW in specific domains, they provide a unified framework that works well across diverse sequential data types without requiring domain-specific preprocessing.

Future work could extend the theoretical analysis to higher smoothness levels and develop adaptive bandwidth selection methods.

\section*{Acknowledgements}
CB and LP gratefully acknowledge funding by Deutsche
Forschungsgemeinschaft through SFB TRR 388 Project B03. The work of DG was supported by the LMBayes project (Linguistic Meaning and Bayesian Modelling).

\bibliography{biblio}

\begin{thebibliography}{63}
\providecommand{\natexlab}[1]{#1}
\providecommand{\url}[1]{\texttt{#1}}
\expandafter\ifx\csname urlstyle\endcsname\relax
  \providecommand{\doi}[1]{doi: #1}\else
  \providecommand{\doi}{doi: \begingroup \urlstyle{rm}\Url}\fi

\bibitem[Abanda et~al.(2019)Abanda, Mori, and Lozano]{abanda2019review}
A.~Abanda, U.~Mori, and J.~A. Lozano.
\newblock A review on distance based time series classification.
\newblock \emph{Data Mining and Knowledge Discovery}, 33\penalty0 (2):\penalty0 378--412, 2019.

\bibitem[Aneiros et~al.(2022)Aneiros, Novo, and Vieu]{aneiros2022variable}
G.~Aneiros, S.~Novo, and P.~Vieu.
\newblock Variable selection in functional regression models: A review.
\newblock \emph{Journal of Multivariate Analysis}, 188:\penalty0 104871, 2022.

\bibitem[Arous(1989)]{arous1989flots}
G.~B. Arous.
\newblock Flots et s{\'e}ries de taylor stochastiques.
\newblock \emph{Probability Theory and Related Fields}, 81\penalty0 (1):\penalty0 29--77, 1989.

\bibitem[Bagnall et~al.(2018)Bagnall, Dau, Lines, Flynn, Large, Bostrom, Southam, and Keogh]{bagnall2018uea}
A.~Bagnall, H.~A. Dau, J.~Lines, M.~Flynn, J.~Large, A.~Bostrom, P.~Southam, and E.~Keogh.
\newblock The {UEA} multivariate time series classification archive, 2018.
\newblock \emph{arXiv preprint arXiv:1811.00075}, 2018.

\bibitem[Baudoin et~al.(2020)Baudoin, Feng, and Ouyang]{baudoin2020density}
F.~Baudoin, Q.~Feng, and C.~Ouyang.
\newblock Density of the signature process of f{B}m.
\newblock \emph{Transactions of the American Mathematical Society}, 373\penalty0 (12):\penalty0 8583--8610, 2020.

\bibitem[Bayer et~al.(2025{\natexlab{a}})Bayer, Pelizzari, and Schoenmakers]{bayer2025primal}
C.~Bayer, L.~Pelizzari, and J.~Schoenmakers.
\newblock Primal and dual optimal stopping with signatures.
\newblock \emph{Finance and Stochastics}, pages 1--34, 2025{\natexlab{a}}.

\bibitem[Bayer et~al.(2025{\natexlab{b}})Bayer, Pelizzari, and Zhu]{bayer2025pricing}
C.~Bayer, L.~Pelizzari, and J.-J. Zhu.
\newblock Pricing {A}merican options under rough volatility using deep-signatures and signature-kernels.
\newblock \emph{arXiv preprint arXiv:2501.06758}, 2025{\natexlab{b}}.

\bibitem[Bayer et~al.(2025+)Bayer, Gogolashvili, and Pelizzari]{bayer2025nonparametric}
C.~Bayer, D.~Gogolashvili, and L.~Pelizzari.
\newblock Nonparametric rates for rough differential equations.
\newblock 2025+.

\bibitem[Bleistein et~al.(2023)Bleistein, Fermanian, Jannot, and Guilloux]{bleistein2023learning}
L.~Bleistein, A.~Fermanian, A.-S. Jannot, and A.~Guilloux.
\newblock Learning the dynamics of sparsely observed interacting systems.
\newblock In \emph{International Conference on Machine Learning}, pages 2603--2640. PMLR, 2023.

\bibitem[Boedihardjo et~al.(2016)Boedihardjo, Geng, Lyons, and Yang]{boedihardjo2016signature}
H.~Boedihardjo, X.~Geng, T.~Lyons, and D.~Yang.
\newblock The signature of a rough path: uniqueness.
\newblock \emph{Advances in Mathematics}, 293:\penalty0 720--737, 2016.

\bibitem[Bouchaud et~al.(2018)Bouchaud, Bonart, Donier, and Gould]{bouchaud2018trades}
J.-P. Bouchaud, J.~Bonart, J.~Donier, and M.~Gould.
\newblock \emph{Trades, quotes and prices: financial markets under the microscope}.
\newblock Cambridge University Press, 2018.

\bibitem[Cass and Salvi(2024)]{cass2024lecture}
T.~Cass and C.~Salvi.
\newblock Lecture notes on rough paths and applications to machine learning.
\newblock \emph{arXiv preprint arXiv:2404.06583}, 2024.

\bibitem[Chen(1957)]{chen1957integration}
K.-T. Chen.
\newblock Integration of paths, geometric invariants and a generalized {B}aker-{H}ausdorff formula.
\newblock \emph{Annals of Mathematics}, 65\penalty0 (1):\penalty0 163--178, 1957.

\bibitem[Chevyrev and Oberhauser(2022)]{chevyrev2018signature}
I.~Chevyrev and H.~Oberhauser.
\newblock Signature moments to characterize laws of stochastic processes.
\newblock \emph{The Journal of Machine Learning Research}, 23\penalty0 (1):\penalty0 7928--7969, 2022.

\bibitem[Cohen et~al.(2023)Cohen, Lui, Malpass, Mantoan, Nesheim, de~Paula, Reeves, Scott, Small, and Yang]{cohen2023nowcasting}
S.~N. Cohen, S.~Lui, W.~Malpass, G.~Mantoan, L.~Nesheim, A.~de~Paula, A.~Reeves, C.~Scott, E.~Small, and L.~Yang.
\newblock Nowcasting with signature methods.
\newblock \emph{arXiv preprint arXiv:2305.10256}, 2023.

\bibitem[Cuchiero et~al.(2023)Cuchiero, Svaluto-Ferro, and Teichmann]{cuchiero2023signature}
C.~Cuchiero, S.~Svaluto-Ferro, and J.~Teichmann.
\newblock Signature {SDE}s from an affine and polynomial perspective.
\newblock \emph{arXiv preprint arXiv:2302.01362}, 2023.

\bibitem[Fermanian(2021)]{fermanian2021embedding}
A.~Fermanian.
\newblock Embedding and learning with signatures.
\newblock \emph{Computational Statistics \& Data Analysis}, 157:\penalty0 107148, 2021.

\bibitem[Fermanian(2022)]{fermanian2022functional}
A.~Fermanian.
\newblock Functional linear regression with truncated signatures.
\newblock \emph{Journal of Multivariate Analysis}, 192:\penalty0 105031, 2022.

\bibitem[Ferraty and Vieu(2006)]{ferraty2006nonparametric}
F.~Ferraty and P.~Vieu.
\newblock \emph{Nonparametric functional data analysis: theory and practice}.
\newblock Springer, 2006.

\bibitem[Folland(1999)]{folland1999modern}
B.~Folland.
\newblock Modern techniques and their applications.
\newblock \emph{Real Analysis (Pure and Applied Mathematics)}, 1999.

\bibitem[Friz and Hager(2025)]{friz2025expected}
P.~K. Friz and P.~P. Hager.
\newblock Expected signature kernels for {L}évy rough paths.
\newblock \emph{arXiv preprint arXiv:2509.07893}, 2025.

\bibitem[Friz and Hairer(2020)]{friz2020course}
P.~K. Friz and M.~Hairer.
\newblock \emph{A course on rough paths}.
\newblock Springer, 2020.

\bibitem[Friz and Victoir(2010)]{friz2010multidimensional}
P.~K. Friz and N.~B. Victoir.
\newblock \emph{Multidimensional stochastic processes as rough paths: theory and applications}, volume 120.
\newblock Cambridge University Press, 2010.

\bibitem[Graves et~al.(2007)Graves, Liwicki, Bunke, Schmidhuber, and Fern{\'a}ndez]{graves2007unconstrained}
A.~Graves, M.~Liwicki, H.~Bunke, J.~Schmidhuber, and S.~Fern{\'a}ndez.
\newblock Unconstrained on-line handwriting recognition with recurrent neural networks.
\newblock \emph{Advances in neural information processing systems}, 20, 2007.

\bibitem[Guo et~al.(2025)Guo, Wang, Zhang, and Zhao]{guo2025consistency}
X.~Guo, B.~Wang, R.~Zhang, and C.~Zhao.
\newblock On consistency of signature using {L}asso.
\newblock \emph{Operations Research}, 2025.

\bibitem[Hambly and Lyons(2010)]{hambly2010uniqueness}
B.~Hambly and T.~Lyons.
\newblock Uniqueness for the signature of a path of bounded variation and the reduced path group.
\newblock \emph{Annals of Mathematics}, pages 109--167, 2010.

\bibitem[Hannun et~al.(2019)Hannun, Rajpurkar, Haghpanahi, Tison, Bourn, Turakhia, and Ng]{hannun2019cardiologist}
A.~Y. Hannun, P.~Rajpurkar, M.~Haghpanahi, G.~H. Tison, C.~Bourn, M.~P. Turakhia, and A.~Y. Ng.
\newblock Cardiologist-level arrhythmia detection and classification in ambulatory electrocardiograms using a deep neural network.
\newblock \emph{Nature medicine}, 25\penalty0 (1):\penalty0 65--69, 2019.

\bibitem[Horvath et~al.(2023)Horvath, Lemercier, Liu, Lyons, and Salvi]{horvath2023optimal}
B.~Horvath, M.~Lemercier, C.~Liu, T.~Lyons, and C.~Salvi.
\newblock Optimal stopping via distribution regression: a higher rank signature approach.
\newblock \emph{arXiv preprint arXiv:2304.01479}, 2023.

\bibitem[Jaber et~al.(2024)Jaber, G{\'e}rard, and Huang]{jaber2024path}
E.~A. Jaber, L.-A. G{\'e}rard, and Y.~Huang.
\newblock Path-dependent processes from signatures.
\newblock \emph{arXiv preprint arXiv:2407.04956}, 2024.

\bibitem[Kidger et~al.(2019)Kidger, Bonnier, Perez~Arribas, Salvi, and Lyons]{kidger2019deep}
P.~Kidger, P.~Bonnier, I.~Perez~Arribas, C.~Salvi, and T.~Lyons.
\newblock Deep signature transforms.
\newblock \emph{Advances in Neural Information Processing Systems}, 32, 2019.

\bibitem[Kir{\'a}ly and Oberhauser(2019)]{kiraly2019kernels}
F.~J. Kir{\'a}ly and H.~Oberhauser.
\newblock Kernels for sequentially ordered data.
\newblock \emph{Journal of Machine Learning Research}, 20\penalty0 (31):\penalty0 1--45, 2019.

\bibitem[Kloeden and Platen(1991)]{kloeden1991stratonovich}
P.~E. Kloeden and E.~Platen.
\newblock Stratonovich and it{\^o} stochastic taylor expansions.
\newblock \emph{Mathematische Nachrichten}, 151\penalty0 (1):\penalty0 33--50, 1991.

\bibitem[Knapp(1996)]{knapp1996lie}
A.~Knapp.
\newblock \emph{{L}ie groups beyond an introduction}, volume 140.
\newblock Springer, 1996.

\bibitem[Kusuoka and Stroock(1987)]{kusuoka1987applications}
S.~Kusuoka and D.~Stroock.
\newblock Applications of the {M}alliavin calculus, part iii.
\newblock \emph{J. Fac. Sci. Univ. Tokyo Sect IA Math}, 34:\penalty0 391--442, 1987.

\bibitem[Lee and Oberhauser(2023)]{lee2023signature}
D.~Lee and H.~Oberhauser.
\newblock The signature kernel.
\newblock \emph{arXiv preprint arXiv:2305.04625}, 2023.

\bibitem[Lemercier et~al.(2021)Lemercier, Salvi, Cass, Bonilla, Damoulas, and Lyons]{lemercier2021siggpde}
M.~Lemercier, C.~Salvi, T.~Cass, E.~V. Bonilla, T.~Damoulas, and T.~J. Lyons.
\newblock Sig{GPDE}: Scaling sparse gaussian processes on sequential data.
\newblock In \emph{International Conference on Machine Learning}, pages 6233--6242. PMLR, 2021.

\bibitem[Lemercier et~al.(2024)Lemercier, Lyons, and Salvi]{lemercier2024log}
M.~Lemercier, T.~Lyons, and C.~Salvi.
\newblock Log-{PDE} methods for rough signature kernels.
\newblock \emph{arXiv preprint arXiv:2404.02926}, 2024.

\bibitem[Li and Shao(2001)]{li2001gaussian}
W.~V. Li and Q.-M. Shao.
\newblock Gaussian processes: inequalities, small ball probabilities and applications.
\newblock \emph{Handbook of Statistics}, 19:\penalty0 533--597, 2001.

\bibitem[Lian(2012)]{lian2012convergence}
H.~Lian.
\newblock Convergence of nonparametric functional regression estimates with functional responses.
\newblock \emph{Electronic journal of statistics}, 2012.

\bibitem[Lyons and McLeod(2022)]{lyons2022signature}
T.~Lyons and A.~D. McLeod.
\newblock Signature methods in machine learning.
\newblock \emph{arXiv preprint arXiv:2206.14674}, 2022.

\bibitem[Lyons and Victoir(2007)]{lyons2007extension}
T.~Lyons and N.~Victoir.
\newblock An extension theorem to rough paths.
\newblock \emph{Annales de l'IHP Analyse non lin{\'e}aire}, 24\penalty0 (5):\penalty0 835--847, 2007.

\bibitem[Lyons(1998)]{lyons1998differential}
T.~J. Lyons.
\newblock Differential equations driven by rough signals.
\newblock \emph{Revista Matem{\'a}tica Iberoamericana}, 14\penalty0 (2):\penalty0 215--310, 1998.

\bibitem[Meister(2016)]{10.3150/15-BEJ709}
A.~Meister.
\newblock {Optimal classification and nonparametric regression for functional data}.
\newblock \emph{Bernoulli}, 22\penalty0 (3):\penalty0 1729 -- 1744, 2016.
\newblock \doi{10.3150/15-BEJ709}.

\bibitem[Moreno-Pino et~al.(2024)Moreno-Pino, Arroyo, Waldon, Dong, and Cartea]{moreno2024rough}
F.~Moreno-Pino, {\'A}.~Arroyo, H.~Waldon, X.~Dong, and {\'A}.~Cartea.
\newblock Rough transformers: Lightweight and continuous time series modelling through signature patching.
\newblock \emph{Advances in Neural Information Processing Systems}, 37:\penalty0 106264--106294, 2024.

\bibitem[Morley and Lyons(2024)]{morley2024roughpy}
S.~Morley and T.~Lyons.
\newblock Roughpy: streaming data is rarely smooth.
\newblock \emph{Proceedings of the 23rd Python in Science Conference}, 2024.

\bibitem[Morrill et~al.(2020)Morrill, Fermanian, Kidger, and Lyons]{morrill2020generalised}
J.~Morrill, A.~Fermanian, P.~Kidger, and T.~Lyons.
\newblock A generalised signature method for multivariate time series feature extraction.
\newblock \emph{arXiv preprint arXiv:2006.00873}, 2020.

\bibitem[Nadaraya(1964)]{nadaraya1964estimating}
E.~A. Nadaraya.
\newblock On estimating regression.
\newblock \emph{Theory of Probability \& Its Applications}, 9\penalty0 (1):\penalty0 141--142, 1964.

\bibitem[Protter(2005)]{P05}
P.~E. Protter.
\newblock \emph{Stochastic {I}ntegration and {D}ifferential {E}quations}.
\newblock Springer, 2005.

\bibitem[Ree(1958)]{ree1958lie}
R.~Ree.
\newblock Lie elements and an algebra associated with shuffles.
\newblock \emph{Annals of Mathematics}, 68\penalty0 (2):\penalty0 210--220, 1958.

\bibitem[Reizenstein and Graham(2018)]{reizenstein2018iisignature}
J.~Reizenstein and B.~Graham.
\newblock The iisignature library: efficient calculation of iterated-integral signatures and log signatures.
\newblock \emph{arXiv preprint arXiv:1802.08252}, 2018.

\bibitem[Reutenauer(2003)]{reutenauer2003free}
C.~Reutenauer.
\newblock Free {L}ie algebras.
\newblock In \emph{Handbook of algebra}, volume~3, pages 887--903. Elsevier, 2003.

\bibitem[Salkeld(2022)]{salkeld2022small}
W.~Salkeld.
\newblock Small ball probabilities, metric entropy and gaussian rough paths.
\newblock \emph{Journal of Mathematical Analysis and Applications}, 506\penalty0 (2):\penalty0 125697, 2022.

\bibitem[Salvi et~al.(2021)Salvi, Cass, Foster, Lyons, and Yang]{salvi2021signature}
C.~Salvi, T.~Cass, J.~Foster, T.~Lyons, and W.~Yang.
\newblock The signature kernel is the solution of a {G}oursat {PDE}.
\newblock \emph{SIAM Journal on Mathematics of Data Science}, 3\penalty0 (3):\penalty0 873--899, 2021.

\bibitem[Schell and Alaifari(2023)]{schell2023nonparametric}
A.~Schell and R.~Alaifari.
\newblock Nonparametric regression of stochastic processes via signatures.
\newblock \emph{Opt. Express}, 31\penalty0 (5):\penalty0 9052--9071, 2023.

\bibitem[Schirrmeister et~al.(2017)Schirrmeister, Springenberg, Fiederer, Glasstetter, Eggensperger, Tangermann, Hutter, Burgard, and Ball]{schirrmeister2017deep}
R.~T. Schirrmeister, J.~T. Springenberg, L.~D.~J. Fiederer, M.~Glasstetter, K.~Eggensperger, M.~Tangermann, F.~Hutter, W.~Burgard, and T.~Ball.
\newblock Deep learning with convolutional neural networks for eeg decoding and visualization.
\newblock \emph{Human brain mapping}, 38\penalty0 (11):\penalty0 5391--5420, 2017.

\bibitem[Selk and Gertheiss(2023)]{selk2023nonparametric}
L.~Selk and J.~Gertheiss.
\newblock Nonparametric regression and classification with functional, categorical, and mixed covariates.
\newblock \emph{Advances in Data Analysis and Classification}, 17\penalty0 (2):\penalty0 519--543, 2023.

\bibitem[Shang(2014)]{shang2014bayesian}
H.~L. Shang.
\newblock Bayesian bandwidth estimation for a functional nonparametric regression model with mixed types of regressors and unknown error density.
\newblock \emph{Journal of Nonparametric Statistics}, 26\penalty0 (3):\penalty0 599--615, 2014.

\bibitem[Stone(1980)]{stone1980}
C.~J. Stone.
\newblock Optimal rates of convergence for nonparametric estimators.
\newblock \emph{The Annals of Statistics}, 8\penalty0 (6):\penalty0 1348--1360, 1980.

\bibitem[Toth and Oberhauser(2020)]{toth2020bayesian}
C.~Toth and H.~Oberhauser.
\newblock Bayesian learning from sequential data using gaussian processes with signature covariances.
\newblock In \emph{International Conference on Machine Learning}, pages 9548--9560. PMLR, 2020.

\bibitem[T\'{o}th et~al.(2025)T\'{o}th, Oberhauser, and Szab\'{o}]{toth2023random}
C.~T\'{o}th, H.~Oberhauser, and Z.~Szab\'{o}.
\newblock Random fourier signature features.
\newblock \emph{SIAM Journal on Mathematics of Data Science}, 7\penalty0 (1):\penalty0 329--354, 2025.

\bibitem[Watson(1964)]{watson1964smooth}
G.~S. Watson.
\newblock Smooth regression analysis.
\newblock \emph{Sankhy{\=a}: The Indian Journal of Statistics, Series A}, pages 359--372, 1964.

\bibitem[Yang et~al.(2022)Yang, Lyons, Ni, Schmid, and Jin]{yang2022developing}
W.~Yang, T.~Lyons, H.~Ni, C.~Schmid, and L.~Jin.
\newblock Developing the path signature methodology and its application to landmark-based human action recognition.
\newblock In \emph{Stochastic Analysis, Filtering, and Stochastic Optimization: A Commemorative Volume to Honor Mark HA Davis's Contributions}, pages 431--464. Springer, 2022.

\bibitem[Young(1936)]{young1936inequality}
L.~C. Young.
\newblock An inequality of the {H}{\"o}lder type, connected with {S}tieltjes integration.
\newblock \emph{Acta Mathematica}, 1936.

\end{thebibliography}

\clearpage
\appendix
\thispagestyle{empty}



\newpage

\appendix
\section{Convergence Guarantee}\label{sec:appendix}

In this section, we provide the detailed proof of Theorem 1. We begin by stating Bernstein's inequality, which serves as our main probabilistic tool.
 \begin{theorem}[Bernstein's Inequality]
    Let \( \eta_1, \eta_2, \dots, \eta_n \) be independent random variables that satisfy the moment condition \begin{equation}\label{Bernstein_condition}
            \mathbb{E}[|\eta_i-\mathbb{E}[\eta_i]|^k] \leq \frac{1}{2} k! L^{k-2} \sigma^2, \quad \forall k \geq 2,
    \end{equation}
    for some positive \( L > 0 \) and \(\sigma\).
    Then, for any \( \delta \in (0,1) \), with probability at least \( 1 - \delta \),
    \[
    \left|\frac{1}{n} \sum_{i=1}^{n} \eta_i - \mathbb{E}[\eta_i] \right| \leq \sqrt{\frac{2\sigma^2 \log(2/\delta)}{n}} + \frac{L \log(2/\delta)}{n}.
    \] 
    Moment condition holds, in particular, for bounded random variables with bounded variance
    \begin{equation}\label{Bernstein_condition_for_bounded_rv}
            |\eta_i| \leq \frac{L}{2} \text{a.s}, \quad \mathbb{E}[\eta_i^2] \leq \sigma^2.
    \end{equation}
\end{theorem}
The proof strategy follows a standard bias-variance decomposition approach: we first decompose the pointwise risk into bias and variance components, then we apply concentration inequalities to bound the variance terms with high probability. 

Let us fix $h>0$ and consider the estimator \eqref{eq:NW_estimator_regression}  when the number of observations $M$ goes to infinity
\[
F_h(x) = \frac{\int y K(h^{-1}\varrho(x,z)) dP(y,z)}{\int K(h^{-1}\varrho(x,z)) dP_X(z)}.
\]
We also need the following notations
\[
p_M(x) = \frac{1}{M} \sum_{j = 1}^M K(h^{-1}\varrho(x, X^{(j)})), \quad \text{and} \quad p_h(x) = \int K(h^{-1}\varrho(x,z)) dP_X(z).
\]

By simple algebraic manipulations, we have

\begin{align*}
\left|\hat{F}(x) - F(x) \right| & \leq \left|\hat{F}(x) - F_h(x) \right|+ \left|F_h(x) - F(x)\right| \\
& = \left|\frac{1}{p_M(x)}\left(\frac{1}{M}\sum_{i=1}^M Y^{(i)} K(h^{-1}\varrho(x, X_i)) - p_M(x)F_h(x) \right)\right|+ \left| F_h(x)-F(x)\right| \\
& \leq \underbrace{\left|1-\frac{p_h(x)-p_M(x)}{p_h(x)}\right|^{-1} }_{A_1} 
\left(\underbrace{\left|\frac{1}{M}\sum_{i=1}^M Y^{(i)} \frac{K(h^{-1}\varrho(x,X_i))}{ p_h(x)} - F_h(x)\right|}_{A_2} \right. \notag \\
&\quad + \left. \underbrace{\left| \frac{p_h(x)-p_M(x)}{p_h(x)} F_h(x) \right|}_{A_3} \right) 
+ \underbrace{\left|F_h(x) - F(x)\right|}_{B}. \\
\end{align*}
\paragraph{Bound on $B$.}
We have
\begin{align*}
|F_h(x)-F(x)| &= \left|\frac{1}{p_h(x)} \int y K(h^{-1}\varrho(x,z)) dP(y,z) -   F(x) \right| \\
& = \left| \frac{1}{p_h(x)} \int F(z) K(h^{-1}\varrho(x,z)) dP_X(z) -   F(x) \right| \\
& \leq \frac{1}{p_h(x)} \left(\int K(h^{-1}\varrho(x,z)) |F(z) - F(x)|  dP_X(z) \right).
\end{align*}
Applying conditions \eqref{Lipschitz condition} and \eqref{thm:kernel} yields
\[
|F_h(x)-F(x)| \leq \frac{1}{p_h(x)} \int K(h^{-1}\varrho(x,z)) \varrho(x,z) ^{\beta} dP_X(z) \leq h^{\beta} B \phi_x(h).
\]
Note that under the condition \ref{thm:kernel}, we have
\begin{equation} \label{proof: p_h lower bound}
    p_h(x) = \int K(h^{-1} \varrho(x,z)) dP_X(z) \geq b \mathbb{P}[\varrho(X,z) \leq h] = b \phi_x(h)
\end{equation}
This leads to the bound on the bias term
\begin{equation}\label{proof:bias bound}
    |F_h(x)-F(x)| \leq \frac{Bh^{\beta}}{b}.
\end{equation}
\paragraph{Bound on $A_1$.} 
Provided that $|\frac{p_h(x)-p_M(x)}{p_h(x)}|<1/2$ we have
\[
\left|1-\frac{p_h(x)-p_M(x)}{p_h(x)}\right|^{-1} \leq 2. 
\]
So it remains to show that $|\frac{p_h(x)-p_M(x)}{p_h(x)}|<1/2$, which is given by the following proposition.
\begin{proposition}
    Let
    \begin{equation}\label{M_is_big}
        M \geq \frac{16 B \log(6/\delta) }{b \phi_x(h)}.
    \end{equation}
    Then, for $\delta \in (0,1),$ with probability at least $1-\delta/3$
    \begin{equation}
        \frac{|p_h(x)-p_M(x)|}{p_h(x)} \leq \frac{1}{2}.
    \end{equation}
\end{proposition}
\begin{proof}
    To establish the result, we verify the conditions \eqref{Bernstein_condition_for_bounded_rv} of Bernstein’s inequality for the random variables
     \[\eta_i = \frac{K(h^{-1}\varrho(x, X^{(i)}))}{p_h(x)}.\]
     Since the kernel is bounded, we have
    \[
    |\eta_i| \leq \frac{B}{p_h(x)}.
    \]
    For the variance,
    \[
    \mathbb{E}[\eta_i^2] = \int \frac{K(h^{-1}\varrho(x, z))^2}{p^2_h(x)} d P_X(z) \leq \frac{B}{p_h(x)} \leq \frac{B}{b\phi_x(h)},
    \]
    where the last inequality follows from \eqref{proof: p_h lower bound}.
    Applying Bernstein’s inequality, we conclude that with probability at least $1-\delta/3$  
    \begin{equation}\label{bound_on_p_h-p_M}
        \frac{|p_h(x)-p_M(x)|}{p_h(x)} \leq \sqrt{\frac{2B \log(6/\delta)}{b \phi_x(h) M}} + \frac{2 B \log(6/\delta)}{b \phi_x(h) M}.
    \end{equation}
    The proposition follows from \eqref{M_is_big}.
\end{proof}
\paragraph{Bound on $A_3$.} Since $|Y| \leq R,$ we have $|F_h(x)| \leq R.$ Therefore, using \eqref{bound_on_p_h-p_M} we have
\begin{equation}\label{bound_A_3}
    \left|\frac{p_h(x)-p_M(x)}{p_h(x)} F_h(x)\right| \leq R \left(\sqrt{\frac{2 B  \log(6/\delta)}{b \phi_x(h) M}} + \frac{2 B  \log(6/\delta)}{b \phi_x(h) M}\right).
\end{equation}
\paragraph{Bound on $A_2$.} The bound follows from the following
\begin{lemma}
    Let $|Y| \leq R.$ Then, with probability at least $1-\delta/3$
    \begin{equation}\label{bound_A_2}
        \left|\frac{1}{M}\sum_{i=1}^M Y^{(i)} \frac{K(h^{-1}\varrho(x,X^{(i)}))}{ p_h(x)} - F_h(x) \right| \leq R\sqrt{\frac{2 B \log(6/\delta)}{b\phi_x(h) M}} + \frac{2R B \log(6/\delta)}{b\phi_x(h) M}.
    \end{equation} 
\end{lemma}
\begin{proof}
    We check the Bernstein condition \eqref{Bernstein_condition_for_bounded_rv} for $\eta_i = Y^{(i)} \frac{K(h^{-1}\varrho(x,X^{(i)}))}{ p_h(x)}$. We have
    \begin{align*}
        |\eta_i| \leq \frac{R B}{b \phi_x(h)}, \quad \mathbb{E}[|\eta_i|^2] \leq \frac{R^2 B}{b\phi_x(h)}.
    \end{align*}
    The bound \eqref{bound_A_2} follows from the Bernstein inequality. 
\end{proof}

\begin{proof}[proof of Theorem \ref{main_theorem}]
    Applying the union bound, we get, with probability at least $1-\delta$
\begin{align*}
    |\hat{F}(x) - F_h(x)| &\leq 2(R + |F_h(x)|) \left( \sqrt{\frac{2 B \log(6/\delta)}{b \phi_x(h) M}} + \frac{2 B \log(6/\delta)}{b \phi_x(h) M} \right) \\
    &\leq 8 R \sqrt{\frac{2 B \log(6/\delta)}{b \phi_x(h) M}}.
\end{align*}
The last inequality, together with the bias bound \eqref{proof:bias bound}, finishes the proof.
\end{proof}

\section{Signature appendix}
In this section, we provide a supplementary introduction to path signatures, complementing Section~\ref{sec:local_sig_regr}, and present the proofs of the convergence rates for local signature regression. Our primary references for signatures and rough paths are the classical monograph \cite{friz2010multidimensional} and the recent lecture notes \cite{cass2024lecture}, to which we refer for further details.

\subsection{Signatures and tensor algebras}\label{app:signature_intro}

As anticipated in Section~\ref{sec:convergence:analysis}, the signature (or path signature) of a continuous path $x:[0,T]\to\mathbb{R}^d$ is the collection of \emph{iterated integrals} against itself. To give a meaning to this object, one needs a suitable notion of regularity for paths $x:[0,T]\to\mathbb{R}^d$.

\begin{definition}\label{def:p_var}
    For any real number $p\geq 1$ and continuous path $x:[0,T] \rightarrow \mathbb{R}^d$, we define the $p$-variation by\begin{equation}\label{eq:p_var_norm}
        \Vert x \Vert_{p-var}= \left ( \sup_{\mathcal{P}\subset [0,T]} \sum_{[u,v]\in \mathcal{P}} |x_v-x_u |^p\right)^{1/p}, 
    \end{equation} where $|\cdot|$ is the Euclidean norm on $\R^d$, and the supremum is taken over all partitions $\mathcal{P}$ of $[0,T]$. We denote by $C^{p-var}([0,T],\R^d)$ the spaces of all continuous paths of finite $p$-variation, that is $\Vert x \Vert_{p-var}<\infty$.
\end{definition}

It is well known that for any $1 \leq p \leq p' < \infty$ one has the inclusion $C^{p'\text{-var}} \subset C^{p\text{-var}}$ \cite[Proposition~5.3]{friz2010multidimensional}. In particular, every continuously differentiable path $x \in C^1$ has finite $1$-variation with
\[
   \|x\|_{1\text{-var}} = \int_0^T |\dot{x}_t| \, dt,
\]
see \cite[Proposition~1.27]{friz2010multidimensional}. Increasing $p \geq 1$ enlarges the class $C^{p\text{-var}}$, admitting more irregular paths, such as $\lfloor 1/p \rfloor$-Hölder paths \cite[Proposition~5.2]{friz2010multidimensional}.

While the class of $p$-variation paths provides the correct analytical framework to define signatures, let us now turn to the algebraic aspects. For multidimensional paths $x$, iterated integrals naturally appear as tensors
\[
   \left(\int_0^T dx_t^{i}\right)_{i \in [d]} \in \mathbb{R}^d, 
   \qquad 
   \left(\int_0^T \int_0^t dx_r^i \, dx_t^j \right)_{i,j \in [d]} \in \mathbb{R}^d \otimes \mathbb{R}^d, 
   \quad \dots
\]
recalling the index notation $[d] = \{1,\dots,d\}$. More generally, the $n$-fold iterated integrals take values in $(\mathbb{R}^d)^{\otimes n}$.  

Let $\{e_1,\dots,e_d\}$ denote the canonical basis of $\mathbb{R}^d$. For any word $w = i_1 \cdots i_n$ with letters from the alphabet $\mathcal{A} = [d]$, we denote by $e_w = e_{i_1} \otimes \cdots \otimes e_{i_n}$ the corresponding basis element of $(\mathbb{R}^d)^{\otimes n}$. Starting from the representation $x = \sum_{i=1}^d e_i x^i$, one can then write the $n$-fold iterated integral tensor as
\[
   \left( \int_0^T \int_0^{t_n} \cdots \int_0^{t_2} dx_{t_1}^{i_1} \cdots dx_{t_n}^{i_n} \right)_{i_1,\dots,i_n \in [d]}
   = \sum_{w = i_1 \cdots i_n}  \left (\int_0^T \int_0^{t_n} \cdots \int_0^{t_2} dx_{t_1}^{i_1} \cdots dx_{t_n}^{i_n}\right )e_w.
\]
The full signature of a path, and its truncation at level $N$, take values in the spaces
\[
   \mathcal{T} = \prod_{k \geq 0} (\mathbb{R}^d)^{\otimes k},
   \qquad 
   \mathcal{T}^{\leq N} = \prod_{k=0}^{N} (\mathbb{R}^d)^{\otimes k},
\]
where we adopt the convention $(\mathbb{R}^d)^{\otimes 0} = \mathbb{R}$.  
Using the basis representation of tensors introduced above, any element $\mathbf{a} \in \mathcal{T}$ can be represented through the series
\[
   \mathbf{a} = \sum_{w\in \mathcal{W}}  \mathbf{a}^{w}e_w, \quad \mathbf{a}^w \in \mathbb{R},
\] where $\mathcal{W}$ denotes the space of all words. Additionally, we denote by $\mathbf{a}^{(k)}$ the projection of $\mathbf{a}$ to the tensor-level $k$, that is $$\mathbf{a}^{(k)}= \sum_{w \in \mathcal{W}^{(k)}}\mathbf{a}^we_w, \qquad \mathcal{W}^{(k)}=\{w=i_1\cdots i_k: i_1,\dots,i_k \in [d]\}\subset \mathcal{W}.$$ Finally, we equip $\mathcal{T}$, as well as its truncated version, with the product
\[
   \otimes : \mathcal{T} \times \mathcal{T} \to \mathcal{T}, \qquad 
   \left( \sum_{w \in \mathcal{W}}   \mathbf{a}^{w} e_w\right) 
   \otimes 
   \left( \sum_{w \in \mathcal{W}}   \mathbf{b}^{w}e_w\right) 
   = \sum_{w \in \mathcal{W}}  \left( \sum_{l=0}^{|w|} \mathbf{a}^{w_1 \cdots w_l} \mathbf{b}^{w_{l+1} \cdots w_{|w|}} \right)e_w,
\]
which turns $(\mathcal{T}, \otimes)$ into an algebra, often called the \emph{extended tensor algebra}; see \cite[Chapter~1.1.2]{cass2024lecture}. Moreover, we can also naturally define addition \[
+ : \mathcal{T} \times \mathcal{T} \to \mathcal{T}, \qquad 
   \left( \sum_{w \in \mathcal{W}}  \mathbf{a}^{w}e_w  \right) 
   + 
   \left( \sum_{w \in \mathcal{W}}   \mathbf{b}^{w} e_w\right) 
   = \sum_{w \in \mathcal{W}}  (\mathbf{a}^w+\mathbf{b}^w)e_w.
\]

Finally, we can endow $\mathcal{T}$ with a Hilbert space structure by using the inner product on $(\mathbb{R}^d)^{\otimes k}$ defined by
\[
   \langle v, w\rangle_{(\mathbb{R}^d)^{\otimes k}}
   = \prod_{l=1}^k \langle v_l, w_l \rangle_{\mathbb{R}^d},
   \qquad 
   v = v_1 \otimes \cdots \otimes v_k, \; w = w_1 \otimes \cdots \otimes w_k,
\] and set $\Vert v \Vert_{(\R^d)^{\otimes k}}= \sqrt{\langle v,v\rangle_{(\R^d)^{\otimes k}}}$. This extends to $\mathcal{T}$ via
   $\langle \mathbf{a}, \mathbf{b} \rangle_{\mathcal{T}}
   = \sum_{k \geq 0} \langle \mathbf{a}^{(k)}, \mathbf{b}^{(k)} \rangle_{(\mathbb{R}^d)^{\otimes k}}$,
which induces the Hilbert space
\[
   \mathfrak{T}
   = \left\{ \mathbf{a} \in \mathcal{T} : 
   \|\mathbf{a}\|
   = \sqrt{\langle \mathbf{a}, \mathbf{a} \rangle_{\mathcal{T}}} < \infty \right\}\subset \mathcal{T},
\]
see \cite[Section~1.1.2]{cass2024lecture} for further details.

We are now ready to defined the signature which dates back to \cite{chen1957integration}, introduced here as a mapping from $p$-variation spaces into the extended tensor algebra. As in the case of continuously differentiable paths $\mathcal{X}=C^1([0,T],\mathbb{R}^d)$ discussed in Section~\ref{sec:local_sig_regr}, the signature can immediately be define for $p$-variation paths with $1 \leq p < 2$. This is made possible by Young’s generalization of the Riemann–Stieltjes integral \cite{young1936inequality}; for further background we refer to \cite[Chapter~6]{friz2010multidimensional}. 

\begin{definition}\label{def:sig_appendix}
    For any real number $1\leq p < 2$ and word $w=i_1\cdots i_n \in \mathcal{W}$, we define \begin{equation}\label{eq:sig_elements_app}
        \Sig{\cdot}^{w}:C^{p-var}([0,T],\mathbb{R}^d) \rightarrow \R, \quad x \mapsto \Sig{x}^{w}=\int_0^T \int_0^{t_n} \cdots \int_0^{t_2} dx_{t_1}^{i_1} \cdots dx_{t_k}^{i_k},
    \end{equation} and the $k$-th signature level is then by $\Sig{x}^{(k)}= \sum_{w\in \mathcal{W}^{(k)}} \Sig{x}^we_w$. Finally, the full signature is defined as \begin{equation}
        \mathrm{Sig}: C^{p-var}([0,T],\R^d) \rightarrow \mathcal{T}, \quad x \mapsto \Sig{x} = \sum_{w\in \mathcal{W}}\Sig{x}^we_w.
    \end{equation}
\end{definition}

In the case of $p$-variation paths with $p>2$, it is no longer clear whether \eqref{eq:sig_elements_app} is well defined, and more sophisticated constructions are required to introduce the \emph{rough path signature}. To keep this introduction concise, we only state the following remark and refer the interested reader to the cited references for details.

\begin{remark}\label{rem:rough_case}
For more irregular paths, such as Brownian sample paths where $p>2$, Young’s extension of the Riemann–Stieltjes integral is no longer sufficient to define the signature. As already observed by Young \cite{young1936inequality}, $\alpha > 1/2$ (equivalently $p < 2$) is necessary and sufficient for a well-defined notion of the integral for Hölder paths.  
One of the major achievements of Lyons’ rough path theory \cite{lyons1998differential} was to realize that the first $\lfloor p \rfloor$ signature levels of $x$ contain exactly the missing information needed to extend integration to irregular paths. By enhancing $x$ to a \emph{rough path}
\[
   x \;\rightsquigarrow\; \mathbf{x} = \big(\Sig{x}^{(1)}, \dots, \Sig{x}^{(\lfloor p \rfloor)}\big),
\]
where these abstract components mimic the classical signature, a consistent integration theory can be developed with respect to $\mathbf{x}$. This provides a robust and deterministic framework for differential equations driven by Brownian motion and more general stochastic processes.  
In particular, Lyons’ extension theorem \cite[Theorem~2.2.1]{lyons1998differential} ensures that the signature of a rough path $\mathbf{x}$ is again well defined and enjoys the same algebraic properties as in Definition~\ref{def:sig_appendix}. Moreover, for a large class of stochastic processes the lift $x \mapsto \mathbf{x}$ is well understood; for instance, semimartingales can be lifted using Itô calculus. We refer to \cite{lyons2007extension, friz2010multidimensional, cass2024lecture} for further details.
\end{remark}

One of the most fundamental properties of the signature, which we rely on multiple times in this article, is the fast decay of $\Sig{x}^{(k)}$ as $k$ increases. For the case $p=1$, the following lemma is an elementary exercise; see \cite[Proposition~1.2.3]{cass2024lecture}, and for more irregular paths see  \cite[Section~9.1.1]{friz2010multidimensional}.

\begin{lemma}\label{lemma:decay}
   For any $x \in C^{1\text{-var}}([0,T],\mathbb{R}^d)$ we have 
   \[
      \Vert \Sig{x}^{(k)} \Vert_{(\R^d)^{\otimes k}} \leq \frac{\|x\|_{1\text{-var}}^k}{k!}, 
      \qquad \forall k \in \mathbb{N}.
   \]
   In particular, $\Sig{C^{1-var}} \subset \mathfrak{T}$.
\end{lemma}

The above lemma ensures, in particular, that the signature semi-distance in \eqref{eq:sig_distance},
\begin{equation*}
   \varrho^{\mathrm{Sig}} : \mathcal{X} \times \mathcal{X} \to \mathbb{R}_+,
   \qquad 
   \dsig{x}{y} = \|\Sig{x} - \Sig{y}\|,
\end{equation*}
as well as its truncated version in \eqref{eq:trunc_distance}, are well defined. 
It is important to note, however, that $\varrho^{\mathrm{Sig}}$ is in general not a true metric, as the following lemma demonstrates.

\begin{lemma}
Let $x \in C^{p-var}$ with $1\leq p <2$ and $\tau:[0,T] \rightarrow [0,T]$ a continuous, non-decreasing surjection. Then \[
\dsig{x}{x\circ \tau}=0.
\]
\end{lemma} This is a direct consequence of the invariance of the signature under time reparametrization; see \cite[Proposition~7.10]{friz2010multidimensional}. The corresponding equivalence classes of paths
\[
   [x] = \{ y \in C^{p\text{-var}} : \dsig{x}{y} = 0 \}, \qquad x \in C^{p\text{-var}},
\]
are well understood and known as \emph{tree-like equivalence}. This was established in \cite{hambly2010uniqueness} for paths of bounded variation and extended to $p>1$ in \cite{boedihardjo2016signature}; we refer to the latter references for a precise definition.  

To obtain a genuine distance, one natural approach is to shrink the space by identifying tree-like equivalent paths $x \sim y$, i.e. by working on the quotient $\mathcal{X}/_\sim$. The alternative considered in this paper is to augment all paths with time, namely
\begin{equation}\label{eq:augmented_space}
   \mathcal{X} = \widehat{C}^{p\text{-var}}([0,T],\mathbb{R}^{d+1})
   = \big\{t \mapsto (t,x_t) : x \in C^{p\text{-var}}([0,T],\mathbb{R}^d) \big\}.
\end{equation} Finally we note that signatures are by construction invariant with respect to the initial condition, that is $\dsig{x}{x+c}=0$, so that $\varrho_\mathrm{Sig}$ only defines a true metric for paths with identical initial condition.

\begin{lemma}\label{lem:metric_appendix}
    For any $1\leq p <2$ and $\xi \in \mathbb{R}^d$, $\varrho_{\mathrm{Sig}}$ defines a true metric on $\mathcal{X}_\xi= \widehat{C}^{p-var} \cap \{\hat{x}:\hat{x}_0 = (0,\xi)\}$.
\end{lemma}
\begin{proof}
    Since $(\mathfrak{T},\Vert \cdot \Vert)$ is a Hilbert space, symmetry and the triangle inequality follow immediately. Moreover, we have $\dsig{x}{y}\geq 0$, and therefore we are left to prove that $\dsig{x}{y}=0$ $\Rightarrow $ $x=y$. Now we can notice that for $\hat{x},\hat{y}\in \widehat{C}^{p-var}$ and any word $w=i1\cdots 1$ it follows by the Cauchy formula for repeated integration that \[
    \Sig{\hat{x}}^{w} = \int_0^T \int_0^{t_n} \cdots \int_0^{t_3}(x^{i}_{t_2}-x^i_0)dt_2\cdots dt_n = \int_0^T(x_t^{i}-x^i_0)\frac{(T-t)^{n-1}}{(n-1)!}dt,
    \] where $n$ is the number of $1$ in $w$, and $i \in [d]$. But then in particular, for all $i \in [d]$ we have
    
    \[\dsig{\hat{x}}{\hat{y}}=0 \Leftrightarrow \Sig{\hat{x}}-\Sig{\hat{y}}=0 \quad \Longrightarrow \quad \int_0^T(x_t^i-y_t^i){(T-t)^n}dt - \frac{(y^i_0-x^i_0)T^n}{n!}= 0,
    \] for all $n\geq 0$. Since $x_0=y_0$ on $\mathcal{X}_\xi$ and since monomials are dense in $L^p$, the latter is only possible if $x^i=y^i$ a.e. for all $i \in [d]$, and thus, in particular, $\hat{x}=\hat{y}$ almost everywhere. 
\end{proof}

From an algebraic perspective, one of the key features of the signature is the \emph{shuffle identity}, which plays an important role in our theoretical results. To this end, it is useful to introduce the notation of pairing words in $\mathcal{W}$ with elements in $\mathcal{T}$ \[
\langle \cdot, \cdot \rangle : \mathcal{W} \times \mathcal{T} \rightarrow \mathbb{R}, \quad (w,\mathbf{a}) \mapsto \langle w,\mathbf{a}\rangle = \mathbf{a}^{w},
\] which extends linearly to the span of words in the alphabet $\mathcal{A}$, that is, to the free associative algebra $\mathbb{R}\langle \mathcal{A} \rangle$. The shuffle identity is a generalization of integration by parts to higher order iterated integrals appearing in the signature. Starting with level $2$ and assuming $x_0=0$ for simplicity, integration by parts suggests that \[
\int_0^Tx_t^idx^j_t+\int_0^Tx_t^jdx^i_t = x^i_tx^j_t \quad \text{that is,} \quad \langle ij+ji, \Sig{x} \rangle = \langle i, \Sig{x}\rangle \langle j, \Sig{x}\rangle.
\] 
To capture this relation on higher levels of the signature, we introduce the shuffle-product on the space of words recursively by \[
w \, \shuffle \,  \emptyset = \emptyset  \, \shuffle \,  w = w, \qquad wi  \, \shuffle \,  vj = (w \, \shuffle \,  vj)i+ (wi  \, \shuffle \,  v)j,
\] which bi-linearly extends to the span of words $\mathbb{R}\langle \mathcal{A}\rangle$, so that $\, \shuffle \,: \mathbb{R}\langle \mathcal{A}\rangle \times \mathbb{R}\langle \mathcal{A}\rangle \rightarrow \mathbb{R}\langle \mathcal{A}\rangle$. The integration by parts identity above then simply reads $\langle i \, \, \shuffle \, \, j , \Sig{x} \rangle = \langle i, \Sig{x} \rangle \langle j, \Sig{x} \rangle$. Perhaps surprisingly, and first observed already in \cite{ree1958lie}, is that this relation holds for arbitrary linear combinations of words \[
\langle w \, \shuffle \, v, \Sig{x} \rangle = \langle w, \Sig{x} \rangle \langle v, \Sig{x} \rangle, \quad \forall w,v \in \mathbb{R}\langle \mathcal{A} \rangle.
\] In particular, the signature, resp. truncations thereof, take values in the following subspaces of $\mathcal{T}$ \[
\mathcal{G} = \{ \mathbf{a} \in \mathcal{T}\setminus \{\mathbf{0} \}: \langle w  \, \shuffle \,  v, \mathbf{a} \rangle = \langle w, \mathbf{a} \rangle \langle v, \mathbf{a} \rangle, \, \, \forall w,v \in \mathbb{R}\langle \mathcal{A} \rangle \}, \quad \mathcal{G}^{\leq N} = \{\mathbf{g} \in \mathcal{G}: \mathbf{g}^{n}= 0, \, \, \forall n >N\},
\] which are often called \emph{group-like elements}. It can for instance be found in \cite[Section 7.3.1]{friz2010multidimensional} that $\mathcal{G}^{\leq N}$ is a Lie group associated to the free Lie algebra $\mathfrak{g}^{\leq N}\in \mathcal{T}^{\leq N}$ with bracket given by $[\mathbf{a},\mathbf{b}] = \mathbf{a}\otimes \mathbf{b}-\mathbf{b} \otimes \mathbf{a}$, with exponential and logarithmic maps given by \[
\exp_{\otimes}: \mathfrak{g} \rightarrow \mathcal{G}, \qquad \mathbf{g} \mapsto \exp_\otimes(\mathbf{g})= \sum_{n\geq 0}\frac{\mathbf{g}^{\otimes n}}{n!}, \qquad \log_\otimes: \mathcal{G} \rightarrow \mathfrak{g}, \quad \mathbf{g} \mapsto \sum_{n\geq 1} (-1)^{n+1}\frac{\mathbf{g}^{\otimes n}}{n!}
\] For a more general introduction to free Lie algebras we refer to \cite{reutenauer2003free}.

In Section \ref{sec:local_sig_regr}, we made the assumption that the stochastic process $X \in \mathcal{X}$ has the property that its truncated signature $\mathrm{Sig}(X)^{\leq N} \in \mathcal{G}^{\leq N}$ admits a density with respect to the Haar measure, which we now define.

\begin{definition}\label{def:haar_measure_lie}
    We denote by $m^N$ (resp. $\mu^N$) the unique\footnote{which exists and is unique up to constant factors on any locally compact group, see, e.g., \cite[Theorem 11.8]{folland1999modern}.} left- and right-invariant Haar measure\footnote{that is, a regular Borel measure $m$, such that $m(gH) = m(H) =m(Hg)$ for all group elements $g$ and Borel measurable sets $H$} on the Lie algebra $\mathfrak{g}^{\leq N}$ (resp. the Lie group $\mathcal{G}^{\leq N}$).
\end{definition}

An important observation is that $\mu^{N}$ is determined by $m^{N}$ through the logarithmic map, we refer to \cite[Proposition 16.40]{friz2010multidimensional} for a proof.

\begin{lemma}\label{lem:Haar}

    The Haar measure $m^N$ coincides with the Lebesgue measure on $\mathfrak{g}^{\leq N}$, and $\mu^N$ on the Lie group $\mathcal{G}^{\leq N}$ is given by the push-forward \[
    \mu^N(A) = m^N(\log_{\otimes}(A)), \quad \forall A \in \mathcal{B}_{\mathcal{G}^{\leq N}}.
    \]
\end{lemma}

We conclude this introduction with the construction of the robust signature, introduced in \cite{chevyrev2018signature} and frequently used in this work. The unbounded nature of the classical signature leads to several theoretical and practical difficulties, as outlined in the main body of this article. To address this issue, Chevyrev and Oberhauser proposed to \emph{bound} the signature map via a tensor normalization
\begin{equation}\label{eq:tensor_normalization}
   \Lambda : \mathcal{T}_1 \to \{ \mathbf{a} \in \mathcal{T}_1 : \|\mathbf{a}\| \leq R \}, 
   \qquad R > 0,
\end{equation} where $\mathcal{T}_1 = \{\mathbf{a}\in\mathcal{T}: \mathbf{a}^{\emptyset}=1\}$,
and defined the robust signature as the composition $\Lambda \circ \operatorname{Sig}$.  
The main challenge in this construction is to ensure that the resulting feature map $x \mapsto \Lambda \circ \operatorname{Sig}(x)$ retains the key advantages of classical signatures, such as their expressivity. A natural way to construct such a normalization is through dilations,
\[
   \delta_\lambda : \mathcal{T}_1 \to \mathcal{T}_1, 
   \qquad 
   \mathbf{a} \mapsto \delta_\lambda(\mathbf{a}) = \sum_{w \in \mathcal{W}} \lambda^{|w|}\mathbf{a}^w e_w , 
   \quad \lambda \in \mathbb{R}_+,
\] where $|w|$ defines the length of the word, that is $|i_1\cdots i_n| =n$. Of course, setting $\Lambda(\mathbf{a}) = \delta_\lambda(\mathbf{a})$ for some fixed $\lambda > 0$ does not suffice to bound the signature map (Definition~\ref{def:sig_appendix}) via $\Lambda \circ \operatorname{Sig}$, since one easily verifies that 
\[
   \Lambda \circ \operatorname{Sig}(\lambda^{-1}x) = \operatorname{Sig}(x), 
   \qquad \forall x \in C^{p\text{-var}}.
\]
Hence, the parameter $\lambda$ must depend on the element itself, i.e. $\Lambda(\mathbf{a}) = \delta_{\lambda(\mathbf{a})}(\mathbf{a}),$
which leads to the following definition; see \cite[Section~3.2]{chevyrev2018signature}.

\begin{definition}\label{def:tensor_normalization}
Let $R>0$ be fixed and $\lambda : \mathcal{T}_1 \to \mathbb{R}_+$ a function. The map
\[
   \Lambda(\mathbf{a}) = \delta_{\lambda(\mathbf{a})}(\mathbf{a})
\]
is called a \emph{tensor normalization} if it is continuous\footnote{With respect to the Banach space topology discussed at the beginning of this chapter.} and injective, and if $\|\Lambda(\mathbf{a})\| \leq R$ for all $\mathbf{a} \in \mathcal{T}_1$.  
In this case, for any $1 \leq p < 2$, the \emph{robust signature} is defined by
\[
   \operatorname{RSig} : C^{p\text{-var}}([0,T],\mathbb{R}^d) \to \mathcal{T}_1, 
   \qquad 
   x \mapsto \operatorname{RSig}(x) = \Lambda(\operatorname{Sig}(x)).
\]
\end{definition}

For our purposes, the most important theoretical property is that the robust signature remains injective on the space $\widehat{C}^{p\text{-var}}$ defined earlier, and in particular
\[
   \varrho_{\operatorname{RSig}}(x,y) = 0 
   \;\;\Longleftrightarrow\;\; x = y, 
   \qquad \forall x,y \in \widehat{C}^{p\text{-var}},
\]
where $\varrho_{\operatorname{RSig}}(x,y) = \|\operatorname{RSig}(x) - \operatorname{RSig}(y)\|.$ In all our numerical experiments we construct $\lambda$ as proposed in \cite[Example~4]{chevyrev2018signature}, which we shall briefly outline now.

\begin{example}\label{ex:tensor_normalization}
    Define the mapping $\Psi=\Psi_{a,C}: [1, \infty) \rightarrow [1,\infty)$ by \[
    \Psi(\sqrt{x})= \begin{cases}
        x & x \leq C \\
        C+C^{1+a}(C^{-a}+x^{-a})/a & x>C,
    \end{cases}
    \] for some fixed constants $a>0$ and $C\geq 1$. Now for any $\mathbf{a} \in \mathcal{T}_1$, we define $\lambda(\mathbf{a})$ to be unique non-negative number such that \[
    \Vert \delta_{\lambda(\mathbf{a})}(\mathbf{a})\Vert^{2}= \sum_{k\geq 0} \lambda(\mathbf{a})^{2k}\Vert \mathbf{a}^{(k)}\Vert_{(\R^d)^{\otimes k}}= \psi(\Vert \mathbf{a} \Vert ).
    \] The resulting $\Lambda = \delta_{\lambda(\cdot)}(\cdot)$ defines a tensor-normalization.
\end{example}

\subsection{Proofs Section \ref{sec:sig_metrics} and \ref{sec:RDE}}\label{app:proofs_signature}
\begin{proof}[of Lemma \ref{lem:small_ball_comp}] First we can notice that for any $x\in \mathcal{X}$, we have \[
\dsig{X}{x}^2= \Vert \Sig{X}-\Sig{x}\Vert^2 =  \dsigtrunc{X}{x}{N}^2+\sum_{k>N} \Vert \Sig{X}^{(k)}-\Sig{x}^{(k)} \Vert^2_{(\R^d)^{\otimes k}} \leq \dsigtrunc{X}{x}{N}^2 \quad \text{a.s.},
\] so that in fact globally it holds that   \[
\mathbb{P}[\dsigtrunc{X}{x}{N}\leq h] \geq \mathbb{P}[\dsig{X}{x}\leq h], \quad \forall x \in \mathcal{X}.
\] For the second part, it follows directly by definition, see also \cite[Proposition 7.8]{friz2010multidimensional}, that $\Sig{x}$ corresponds to the terminal value to the $\mathcal{T}$-valued ODE \[\Sig{x}_0=\mathbf{1} \in \mathcal{T}, \quad 
d\Sig{x}_t=\Sig{x}_t\otimes dx_t, \quad 0<t\leq T,
\]where $\mathbf{1}^{\emptyset }=1$ and $\mathbf{1}^w=0$ for all $w\neq \emptyset$. Following the techniques used in \cite{friz2025expected} for \emph{free developments} in $\mathcal{T}$, it follows from the triangle inequality that for any $x,y\in \mathcal{X}$  \begin{align*}
    \Vert \Sig{x}-\Sig{y} \Vert \leq &  \int_0^T \Vert \Sig{x}_t \otimes \dot{x}_t-\Sig{y}_t \otimes \dot{y}_t \Vert dt \\ \leq  & \int_0^T \Vert (\Sig{x}_t-\Sig{y}_t)\otimes \dot{x}_t \Vert dt+ \int_0^T \Sig{y}_t\otimes |\dot{x}_t-\dot{y}_t|dt \\ \leq & \int_0^T \Vert \Sig{x}_t-\Sig{y}_t \Vert |\dot{x}_t|dt+\int_0^t\Vert \Sig{y}_t \Vert |\dot{x}_t-\dot{y}_t|dt.
\end{align*}
An application of Lemma \ref{lemma:decay} shows that $\int_0^T\Vert \Sig{y}_t\Vert|\dot{x}_t-\dot{y}_t|dt \leq e^{\Vert y \Vert_{1-var}}\Vert x-y \Vert_{1-var} =: \alpha(T)$. Applying Grönwalls inequality together with $\beta(t)= |\dot{x}_t|$, it follows that \begin{align*}
    \Vert \Sig{x}-\Sig{y} \Vert \leq & e^{\Vert y \Vert_{1-var}}\Vert x-y \Vert_{1-var}+ \int_0^Te^{\Vert y \Vert_{1-var;[0,t]}} \Vert {x}-y \Vert_{1-var;[0,t]}|\dot{x}_t|e^{\Vert x \Vert_{1-var;[t,T]}}dt  \\ \leq & \Vert x-y \Vert_{1-var}\left (e^{\Vert y \Vert_{1-var}}+\Vert x \Vert_{1-var}e^{\Vert y\Vert_{1-var}+\Vert x \Vert_{1-var}} \right ).
\end{align*}
Now setting $C_\mathcal{R} =e^{\mathcal{R}}+\mathcal{R}e^{2\mathcal{R}} $, for any random variable $X\in \mathcal{X}_M$ we almost surely have $\dsig{X}{x} \leq C_\mathcal{R}\Vert X-x \Vert_{1-var}$, and therefore \[
\mathbb{P}[\dsig{X}{x}\leq h] \geq \mathbb{P}[\Vert X-x\Vert_{1-var} \leq Ch], \quad \forall x \in \mathcal{X}_\mathcal{R},
\] where $C=C_\mathcal{R}^{-1}$.
\end{proof}

\begin{proof}[of Proposition \ref{prop:truncated_samllball}]
    For any random variable $X\in \mathcal{X}$, such that Assumption \ref{ass:density} holds true, we have \[
    \phi_x^{\mathrm{Sig},N}(h)=\mathbb{P}[\dsigtrunc{X}{x}{N}\leq h] = \int_{\mathcal{G}_{h,x}^{\leq N}}p(\mathbf{g})d\mu^N(\mathbf{g}) \geq c \mu^N(\mathcal{G}_{h,x}^{\leq N}),
    \] where $\mathcal{G}_{h,x}^{\leq N}=\{\mathbf{g} \in \mathcal{G}_1^{\leq N}: \Vert \mathbf{g}-\Sig{x}^{\leq N} \Vert \leq h\}$. An application of Lemma \ref{lem:Haar} together with a change of variable for the push-forward measure shows \[ 
    \mu^N(\mathcal{G}_{h,x}^{\leq K}) = m^N\Big (\left \{\mathbf{g}\in \mathfrak{g}^{\leq N}:  \Vert \exp_\otimes(\mathbf{g})-\exp_\otimes(\mathbf{g_0}(x)) \Vert \leq h\right \}\Big ), \quad \mathbf{g_0}(x)= \log_{\otimes}(\Sig{x}^{\leq N}),
    \] where $\log_\otimes$ is the inverse of $\exp_\otimes$. 
    Now since $\mathcal{G}^{\leq N}$ is a free nilpotent, connected and simply connected Lie group, $\exp_{\otimes}$ is a diffeomorphism \cite[Theorem 1.127]{knapp1996lie}, so that in particular \begin{align*}
        m^N\Big (\left \{\mathbf{g}\in \mathfrak{g}^{\leq N}:  \Vert \exp_\otimes(\mathbf{g})-\exp_\otimes(\mathbf{g_0}(x)) \Vert \leq h\right \}\Big ) \geq & m^N\Big (\left \{\mathbf{g}\in \mathfrak{g}^{\leq N}:  \Vert \mathbf{g}-\mathbf{g_0}(x) \Vert \leq C h\right \}\Big )\\ \sim &   h^{\mathrm{dim}(\mathfrak{g}^{\leq N})},
    \end{align*} since $m^N$ is the Lebesgue measure by Lemma \ref{lem:Haar}. Finally, it follows from \cite[Theorem 6]{reutenauer2003free} that $\mathrm{dim}(\mathfrak{g}^{\leq N})$ is given by $\nu(N)$ in Lemma \ref{prop:truncated_samllball}.
\end{proof}

\begin{proof}[of Corollary \ref{cor:truncated_rates}] Since all the assumption of Theorem \ref{main_theorem} hold, we know that for any $\delta>0$ there exists a constant $C=C(R,\delta)>0$ such that \[
\mathbb{P}\left [|\hat{F}(x)-F(x)| \leq C \left (h^\beta+\sqrt{\frac{1}{\phi_x(h)M}}\right )\right ] \geq 1-\delta.
\] From Proposition \ref{prop:truncated_samllball} we know $\phi_x(h)\geq \tilde{C}h^{\nu(K)}$ for some $\tilde{C}>0$. On the other hand, we easily see that \[ h=M^{-1/(2\beta+\nu(N))}\Longleftrightarrow h^{\beta} = \sqrt{\frac{1}{h^{\nu(N)}M}},
\] and thus for this choice of $h$, we find a new constant $\hat{C}>0$ such that 
    \[\mathbb{P}\left [|\hat{F}(x)-F(x)| \leq C \left (h^\beta+\sqrt{\frac{1}{\phi_x(h)M}}\right )\right ]  \geq 
\mathbb{P}\left [|\hat{F}(x)-F(x)| \leq \hat{C} M^{-\beta/(2\beta + \nu(N))} \right ] \geq 1-\delta,
\] which finishes the proof.
\end{proof}
\begin{proof}[of Theorem \ref{thm:RDE}]
    First, we note that the rough path signature $t \mapsto \mathrm{Sig}(\mathbf{x})_t=\mathrm{Sig}(\mathbf{x}|_{[0,t]})$ -- given by Lyons Extension theorem \cite[Theorem 2.2.1]{lyons1998differential}, uniquely solves the linear rough differential equation (see \cite[Theorem 8.3 and Chapter 8.9]{friz2020course}) on the Hilbert space $(\mathfrak{T},\Vert \cdot \Vert )$  \[
    \mathrm{Sig}(\mathbf{x})_0=\mathbf{1}, \quad d\Sig{\mathbf{x}}_t =\Sig{\mathbf{x}}_t\otimes d\mathbf{x}_t, \quad 0<t \leq T.
    \]
    Since $\mathbf{x} \in \mathcal{X}_R = \mathscr{C}_g^\alpha \cap \{ \vertiii{\mathbf{x}}_{\alpha}\leq R \}$, it follows from \cite[Theorem 8.5]{friz2020course} that the target $Y$ is bounded, 
    and that $\mathrm{Sig}$ is locally Lipschitz, that is \begin{equation*}
        \Vert \Sig{\mathbf{x}}-\Sig{\mathbf{y}} \Vert \leq K \vertiii{\mathbf{x}-\mathbf{y}}_{\alpha;[0,T]},
    \end{equation*} for some constant $K=K(R)>0$. In particular, we have the small-ball probability lower-bound 
    \begin{equation}\label{eq:esti1}
    \begin{aligned}
    \phi^{\varrho}_\mathbf{x}(h) =  \mathbb{P}\left [\Vert \mathrm{Sig}({\mathbf{B}}^R)-\Sig{\mathbf{x}}\Vert \leq h \right ] \geq & \mathbb{P}\left [\{T_R\geq T \} \cap \{\Vert \mathrm{Sig}({\mathbf{B}})-\Sig{\mathbf{x}}\Vert \leq h\} \right ] \\  \geq & \mathbb{P}\left [\{T_R\geq T \} \cap \{ \varrho_\alpha({\mathbf{B}},\mathbf{x})\leq K^{-1}h \} \right] \\ = & \mathbb{P}[\varrho_\alpha({\mathbf{B}},\mathbf{x})\leq K^{-1}h] \times \mathbb{P}[T_R \geq T \, | \,  \varrho_\alpha({\mathbf{B}},\mathbf{x})\leq K^{-1}h] \\ = & \phi_\mathbf{x}^{\varrho_\alpha}(K^{-1}h) \times  \mathbb{P}\left [\vertiii{\mathbf{B}}_{\alpha;[0,T]} \leq R \, | \, \varrho_\alpha({\mathbf{B}},\mathbf{x})\leq K^{-1}h \right ].
    \end{aligned}
    \end{equation}
where we recall the homogeneous rough path metric \cite[Chapter 2.3]{friz2020course} \[
\varrho_\alpha(\mathbf{x},\mathbf{y}):= \sup_{0\leq s\leq t \leq T}\frac{\Vert \mathbf{x}_{s,t}^{-1}\otimes \mathbf{y}_{s,t} \Vert}{|t-s|^\alpha}, 
\] for any homogeneous norm $\Vert \cdot \Vert$ on $\mathcal{T}^{\leq 2}$. Now since $\vertiii{\mathbf{x}}_\alpha < R $, we can choose $h$ such that $ K^{-1}h< R-\vertiii{\mathbf{x}}$, we have $\{\varrho_\alpha(\mathbf{B},\mathbf{x})\leq K^{-1}h\}\subset \{\vertiii{\mathbf{B}}_{\alpha;[0,T]} \leq R\}$ and thus in particular \[
\mathbb{P}\left [\vertiii{\mathbf{B}}_{\alpha;[0,T]} \leq R \, | \, \varrho_\alpha({\mathbf{B}},\mathbf{x})\leq K^{-1}h \right ] =1.
\]

On the other-hand, the following small-ball probability lower bound for such $\mathbf{x} \in \mathscr{H}$ was shown in \cite[Lemma 3.2 and Theorem 4.1]{salkeld2022small} \[
\phi_\mathbf{x}^{\varrho_\alpha}(K^{-1}h) \geq \exp \left (-\frac{\Vert x \Vert_\mathcal{H}^2}{2}\right )\mathbb{P}[\varrho_\alpha(\mathbf{B},\mathbf{1})\leq K^{-1}h] \geq \exp \left (-\frac{\Vert x \Vert_\mathcal{H}^2}{2}\right )\exp \left (-\tilde{K}h^{\alpha-\frac{1}{2}} \right ),
\] where $\tilde{K}=(K^{-1})^{\alpha-1/2}$. Combining these observations with \eqref{eq:esti1}, we conclude \[\phi^{\varrho}_\mathbf{x}(h) \geq C(\mathbf{x}) \times  \exp \left (-\hat{K}h^{\alpha-\frac{1}{2}} \right )
\] where $\hat{K}=(1-a)\tilde{K}$.
Now choosing $h = \left (\frac{(1-\epsilon)\log(M)}{\hat{K}}\right )^{{\alpha-
    1/2}}$ for some $\epsilon>0$ and $M$ large enough we have \[
h^{\beta}+\sqrt{\frac{1}{\phi^\varrho_\mathbf{x}(h)M}} = \mathcal{O}_\mathbf{x}\left (  \left (\frac{(1-\epsilon)\log(M)}{\hat{K}} \right )^{\beta(\alpha-1/2)}+\sqrt{\frac{1}{M^{-(1-\epsilon)}M}} \right )=\mathcal{O}_\mathbf{x}(\log(M)^{-\zeta}),
\] for $\zeta = \beta(1/2-\alpha)$. We can then conclude the proof using Theorem \ref{main_theorem}. 
\end{proof}
\end{document}